\documentclass[lettersize,journal]{IEEEtran}
\usepackage{amsmath,amsfonts}
\usepackage{algorithm}
\usepackage{algpseudocode}  
\usepackage{array}
\usepackage[caption=false,font=normalsize,labelfont=sf,textfont=sf]{subfig}
\usepackage{textcomp}
\usepackage{stfloats}
\usepackage{url}
\usepackage{verbatim}
\usepackage{graphicx}
\usepackage{cite}
\usepackage{hyperref}

\usepackage{booktabs}       
\usepackage{threeparttable} 

\usepackage{xcolor}         

\usepackage{amssymb}
\usepackage{amsthm}
\usepackage{bm}

\newcommand{\R}{\mathbb{R}}

\newtheorem{definition}{Definition}
\newtheorem{assumption}{Assumption}
\newtheorem{theorem}{Theorem}
\newtheorem{lemma}{Lemma}

\newtheorem{remark}{Remark}

\begin{document}
	
	\title{Topology-Independent Robustness of the Weighted Mean under \\ Label Poisoning Attacks in Heterogeneous Decentralized Learning}
	
	\author{Jie~Peng,~
		Weiyu~Li,~
		Stefan~Vlaski,~
		Qing~Ling~
		\thanks{Jie Peng and Qing Ling are both with School of Computer Science
			and Engineering, Sun Yat-Sen University, Guangzhou, Guangdong 510006, China (Corresponding E-mail: lingqing556@mail.sysu.edu.cn). Weiyu Li is with School of Engineering and Applied Science, Harvard University, Cambridge, MA 02138, USA. Stefan Vlaski is with Department of Electrical and Electronic Engineering, Imperial College London, London, SW7 2BT, UK.}
\thanks{Part of this work was conducted while Jie Peng was a visiting PhD student at Imperial College London with support from the China Scholarship Council. Qing Ling (corresponding author) is supported by the National Key R\&D Program of China under Grant 2024YFA1014002 and the NSF China under Grant 62373388.}
	}

	
	
	
	\maketitle
	
	\begin{abstract}
		Robustness to malicious attacks is crucial for practical decentralized signal processing and machine learning systems. A typical example of such attacks is label poisoning, meaning that some agents possess corrupted local labels and share models trained on these poisoned data. To defend against malicious attacks, existing works often focus on designing robust aggregators; meanwhile, the weighted mean aggregator is typically considered a simple, vulnerable baseline. This paper analyzes the robustness of decentralized gradient descent under label poisoning attacks, considering both robust and weighted mean aggregators. Theoretical results reveal that the learning errors of robust aggregators depend on the network topology, whereas the performance of weighted mean aggregator is topology-independent. Remarkably, the weighted mean aggregator, although often considered vulnerable, can outperform robust aggregators under sufficient heterogeneity, particularly when: (i) the global contamination rate (i.e., the fraction of poisoned agents for the entire network) is smaller than the local contamination rate (i.e., the maximal fraction of poisoned neighbors for the regular agents); (ii) the network of regular agents is disconnected; or (iii) the network of regular agents is sparse and the local contamination rate is high. Empirical results support our theoretical findings, highlighting the important role of network topology in the robustness to label poisoning attacks.
	\end{abstract}
	
	
	\section{Introduction}
	With the rise of large-scale foundation models, distributed signal processing and machine learning has become a dominant paradigm for efficient training on massive data using dispersed computation devices \cite{liu2024survey, liang2024communication, xu2025resource}. This paradigm includes two major categories: centralized learning (also known as federated learning) with a server-worker architecture, as well as decentralized learning operating under a peer-to-peer archi- tecture. This paper focuses on the latter, which offers higher flexibility in network topology and more balanced overhead in communication.
	
	In decentralized learning, multiple computation devices (referred to as agents hereafter) form a connected network \cite{tu2012diffusion,nedich2015convergence,lu2021optimal, cao2025on, huang2025dual}. Each agent possesses its own data and maintains its own model, updating the local model with the local data in each iteration and exchanging the updates with its neighbors. Then, each agent aggregates the received messages using a specified aggregator to update the local model. Such an approach has already been successfully applied in signal processing \cite{chang2020distributed, pu2020asymptotic, yemini2025Resilient} and deep learning \cite{tang2018d, yuan2022revisiting}.
	
	
	However, without a central authority, decentralized learning is vulnerable to malicious attacks. Under data poisoning or cyber attacks, some agents might behave abnormally and send incorrect messages to their neighbors, thereby disturbing the training process. For instance, in decentralized training of content moderation models on social media platforms, some users may attempt to manipulate the models by intentionally mislabeling harmful or inappropriate content as safe, leading to poisoned local labels and finally corrupting the learned models. Similar risks also exist in spam detection and crowd-sourcing applications \cite{jha2023label, jebreel2024lfighter}.
	These types of malicious behaviors are commonly modeled as label poisoning attacks, where the local labels of some agents are poisoned.

	While label poisoning attacks have been investigated in centralized learning, its impact on decentralized systems remains unexplored. For the latter, surprisingly we can observe from empirical results that: \textbf{the performance of robust aggregators is affected by the network topology; however that of the weighted mean aggregator is not}. Thus, in certain application scenarios, the weighted mean aggregator, often considered vulnerable, can outperform robust aggregators. Building upon this observation, we make the following theoretical contributions.
	
	\textbf{C1)} This is the \textit{first} work to investigate the robustness of weighted mean aggregator for decentralized learning.
	Our findings reveal a crucial insight: robust aggregators cannot always outperform the weighted mean aggregator under label poisoning attacks. This insight may encourage us to reconsider what are appropriate application scenarios for deploying different aggregators in decentralized learning.
	
	\textbf{C2)} We derive upper and lower bounds on the learning errors under label poisoning attacks for decentralized first-order algorithms. Theoretical results reveal that the learning errors of robust aggregators depend on the network topology, whereas the performance of weighted mean aggregator is topology-independent.

	\textbf{C3)} We provide both theoretical and empirical evidence that the weighted mean aggregator can outperform robust aggregators with sufficient heterogeneity, specifically when: (i) the global contamination rate (i.e., the fraction of poisoned agents for the entire network) is smaller than the local contamination rate (i.e., the maximal fraction of poisoned neighbors for the regular agents); (ii) the network of regular agents is disconnected; or (iii) the network of regular agents is sparse and the local contamination rate is high.
	
	\subsection{Related Works}
	
	Malicious attacks in decentralized learning can be broadly classified as targeted or untargeted. In this paper, we mainly focus on untargeted attacks, under which the goal of poisoned agents is to degrade the performance of trained models. Untargeted attacks fall into two main categories: model poisoning and data poisoning \cite{kairouz2021advances}.
	
	In model poisoning attacks, poisoned agents are able to send arbitrarily malicious local models to their neighbors. To defend against such attacks, many existing works combine the decentralized gradient descent algorithm (also known as the diffusion algorithm \cite{sayed2014adaptation}) with robust aggregators, such as trimmed mean (TriMean) \cite{yin2018byzantine, fang2022bridge}, FABA \cite{xia2019faba}, centered clipping (CC) \cite{karimireddy2021learning}, clipped gossip (CG) \cite{he2022byzantine, han2025byzantine}, IOS \cite{wu2023byzantine}, RFA \cite{pillutla2022robust}, Distance-MinMax Filtering \cite{kuwaranancharoen2024scalable}, M-FTM \cite{cao2025resilient}, to name a few. These aggregators aim to produce models bounded in distance from the global average. To address the issue of high data heterogeneity, which degrades the performance under model poisoning attacks, \cite{peng2024byzantine} introduce a total variation norm penalty into the original optimization problem, achieving an order-optimal learning error. Recent works also establish a unified framework for robust decentralized momentum methods \cite{farhadkhani2023robust} and determine the optimal breakdown points for robust aggregators \cite{gaucher2024unified}. 
	
	In data poisoning attacks, poisoned agents possess fabricated local data, but otherwise adhere to the prescribed learning protocol \cite{alber2025medical,cina2024machine,zhao2025data}. To defend against data poisoning attacks, a number of approaches utilize attack detection and adversarial training for recommendation systems, as well as sybil defense and data aggregation for crowd-sourcing platforms \cite{yuan2017sybil,zhao2021data}. Other methods involve data sanitization \cite{li2019dpif}, which removes the data that significantly deviates from the clean data, and data augmentation \cite{borgnia2021strong}, which adds additional data to regularize the decision boundary. For further coverage, see \cite{kairouz2021advances, fan2022survey}.
	
	It is important to note that general model poisoning and data poisoning attacks allow poisoned agents to arbitrarily corrupt their local models or data, potentially causing unbounded disturbance to the training process. However, real-world attacks are often not necessarily as malicious as the ones described above. For example, works such as \cite{ paudice2019label, tavallali2022adversarial, jha2023label, hallaji2023label} focus on label poisoning attacks, where some agents have poisoned local labels and update their local models with these poisoned data. Other studies address label flipping attacks \cite{xu2022rethinking, jebreel2024lfighter}, a specific form of label poisoning attacks in which  the poisoned agents flip their local labels from source classes to target classes. It is worth mentioning that although label poisoning attacks are special cases of data poisoning attacks, they are far from the worst-case attacks, as the poisoned agents only corrupt their local data at the label level.
	
	Robust aggregators have shown effectiveness against label poisoning attacks \cite{ karimireddy2021learning, he2022byzantine, gorbunov2023variance}, and some have been specifi- cally developed for this purpose. For example, LFighter \cite{jebreel2024lfighter} uses gradient clustering to filter out malicious updates. However, LFighter is designed for centralized learning with low heterogeneity \cite{peng2025mean}, and performs poorly when the data is highly non-i.i.d. and the regular network is disconnected, as shown in our numerical experiments. Beyond robust aggregators, other defenses also exist. For instance, \cite{hallaji2023label} proposes adversarial training with locally generated label-poisoned data, but limited to centralized learning with low heterogeneity too, making it less suitable for decentralized systems \cite{li2025threats}.
	
	Recent studies \cite{shejwalkar2022back, peng2025mean}, similar to our findings, demonstrate the robustness of mean aggregator against label poisoning attacks in centralized learning. However, they focus on the server-worker architecture and do not account for the network topology, an important factor influencing the performance of decentralized algorithms. In contrast, our work considers the more general and challenging decentralized architecture, which subsumes centralized learning as a special case.
	
	In summary, this is the first work to study the robustness of weighted mean aggregator in decentralized learning. Our results demonstrate that robust aggregators may not consistently surpass the weighted mean aggregator under label poisoning attacks, highlighting the need of revisiting the appropriate application scenarios for various aggregators.
	

	\subsection{A Motivating Example}
	Before introducing the technical details, we first illustrate a surprising empirical observation with a motivating example.
	
	We train a softmax regression model using the MNIST dataset through decentralized learning with 10 agents. The dataset has 10 classes of digits, and each agent contains the data from only one class. Among the agents, one of them is poisoned with label flipping attacks that reverses the local labels from $b$ to $9 - b$. We compare the classification accuracies for two network topologies: a two-castle graph and a fan graph. Note that for both network topologies, the regular networks, defined as the networks of regular agents, also differ and have less connectivity than the full networks.

	\begin{figure}[H]
		\centering
		\includegraphics[scale=0.22]{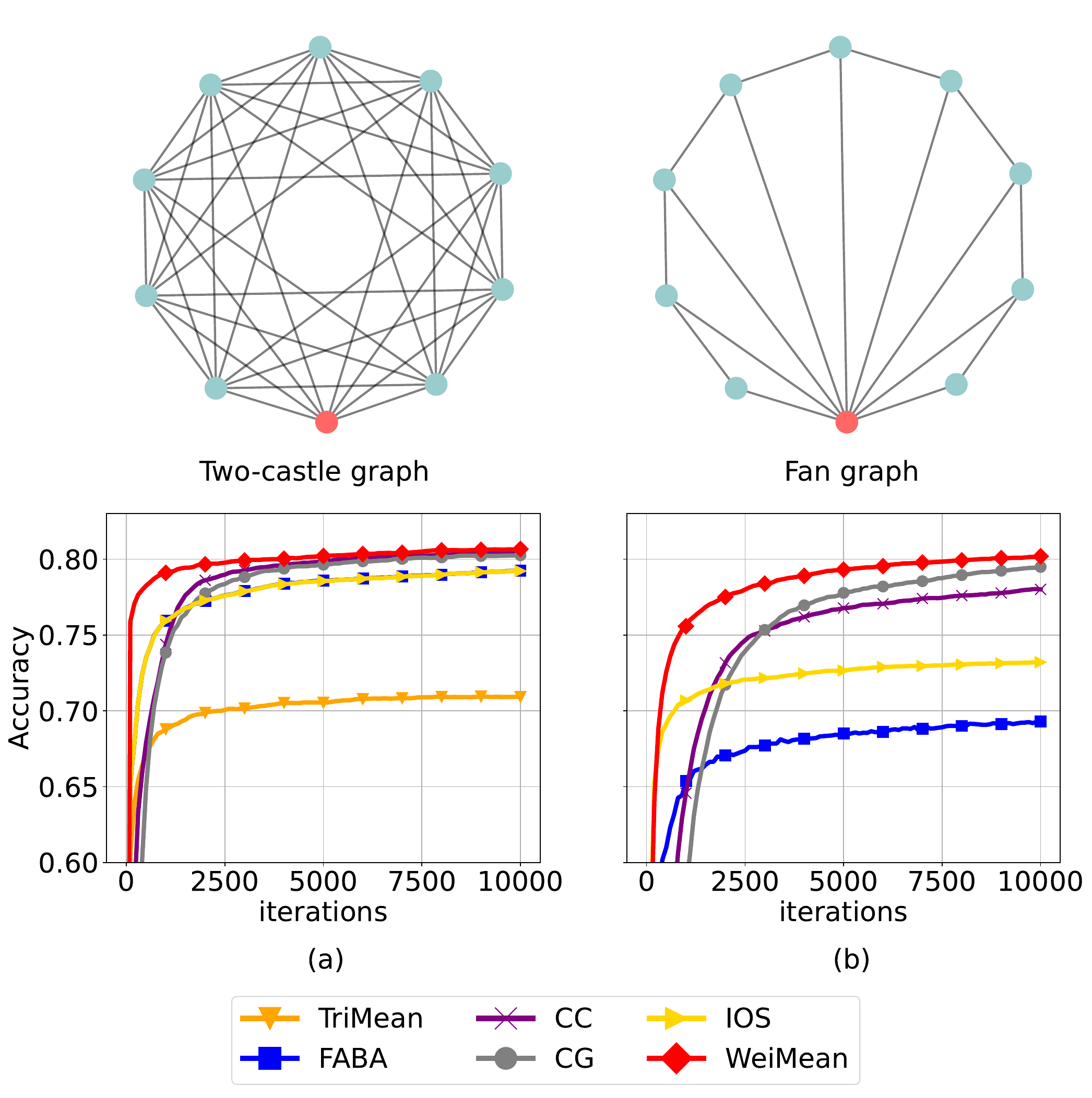}
		\caption{Classification accuracies of the softmax regression model trained on the non-i.i.d. MNIST dataset under label flipping attacks for: (a) the two-castle graph, (b) the fan graph. The blue and red points represent the regular and poisoned agents, respectively. The classification accuracy is in terms of the average model of all regular agents. WeiMean stands for the weighted mean, with the weights constructed using the Metropolis-Hastings rule \cite{he2022byzantine}.}
		\label{fig:SVM_digits_label_flipping_toyexample}
	\end{figure}
	
	Figure \ref{fig:SVM_digits_label_flipping_toyexample} compares the weighted mean aggregator against 5 state-of-the-art robust aggregators (TriMean, FABA, CC, CG and IOS). Surprisingly, the weighted mean remains robust across both network topologies, while the performance of the robust aggregators degrades significantly and falls behind that of weighted mean on the fan graph. This observation suggests that the suitability of different aggregators strongly depends on the network topology, and motivates our central question:
	
	\begin{center}
		\it{\textbf{Under what network topology can the weighted mean aggregator provably outperform robust aggregators?}}
	\end{center}

	\section{Problem Formulation}
	\label{problem formulation}
	
	We consider an undirected connected network $\mathcal{G} = (\mathcal{W}, \mathcal{E})$ of $W=|\mathcal{W}|$ agents, where $\mathcal{W}$ is the set of all agents and $\mathcal{E}$ is the set of all edges. An edge $(w, v) \in \mathcal{E}$ means that agents $w$ and $v$ are neighbors and can communicate with each other. The edge set $\mathcal{E}$ also contains self-loops $(w , w)$ for all agents $w \in \mathcal{W}$. Let $\mathcal{R} \subseteq \mathcal{W}$ be the set of regular agents, with $|\mathcal{R}| = R$. Our goal is to find the minimizer of the following optimization problem
	\begin{align}\label{problem: 1}
		\min_{x \in \R^D} f(x) \triangleq \frac{1}{R}  \sum_{w \in \mathcal{R}} f_w(x).
	\end{align}
	Here, $x\in \R^D$ is the optimization variable (also called as model), $f_w(x)$ is the local cost of regular agent $w \in \mathcal{R}$, $f(x)$ is the global cost that averages the local costs of all regular agents. Denote $\mathcal{R}_w$ and $\mathcal{N}_w$ as the set of regular neighbors and all neighbors of agent $w$ respectively, with $R_w = |\mathcal{R}_w|$ and $N_w = |\mathcal{N}_w |$. Further, let $\bar{\mathcal{R}}_w \triangleq \mathcal{R}_w \cup \{w\}$ and $\bar{\mathcal{N}}_w \triangleq \mathcal{N}_w \cup \{w\}$ denote the sets consisting of agent $w$ with its regular neighbors and with all neighbors, respectively, with $\bar{R}_w = |\bar{\mathcal{R}}_w|$ and $ \bar{N}_w = |\bar{\mathcal{N}}_w|$.
	
	
	With the aim of solving \eqref{problem: 1}, we characterize poisoned agents as those under label poisoning attacks: they follow the algorithmic protocol but possess poisoned local labels. Thus, the only difference between the regular and poisoned agents lies in the corrupted labels, which can affect model updates and harm the learned model.
	
	
	To solve \eqref{problem: 1} in the presence of poisoned agents, most existing works adopt the decentralized gradient descent algorithm with a robust aggregator. Generally, the algorithm operates in three steps: (i) local update; (ii) model exchange; and (iii) model aggregation. In this algorithm, each agent has its own local model, which is updated in every iteration using its local data. Specifically, in the local update step, each agent computes its local gradient and updates its model accordingly. For agent $w \in \mathcal{W}$, we denote $\hat{x}_w^k$ and $\nabla \hat{f}_w(\hat{x}_w^k)$ as its local model and local gradient at iteration $k$, respectively. Then the local update for  agent $w \in \mathcal{W}$ is
	\begin{align}\label{eq: local update}
		\hat{x}_w^{k+\frac12} = \hat{x}_w^k - \gamma^k \cdot \nabla \hat{f}_w(\hat{x}_w^k).
	\end{align}
	where $\gamma^k > 0$ is the step size. In the model exchange step, each agent $w \in \mathcal{W}$ sends the updated local model $\hat{x}_w^{k+\frac12}$ to its neighbors and receives the local models from all of its neighbors. In the model aggregation step, after receiving $\{\hat{x}_v^{k+\frac12} \! : \! v \in \mathcal{N}_w\}$ from its neighbors, each agent $w \in \mathcal{W}$ aggregates them with a robust aggregator $\text{RAgg}(\cdot)$, as
	\begin{align}\label{eq: RAgg}
		\hat{x}_w^{k+1} = \text{RAgg}\left( \{\hat{x}_v^{k+\frac{1}{2}}: v \in \bar{\mathcal{N}_w}\}\right).
	\end{align}
	Robust aggregators such as TriMean, FABA, CC, CG and IOS are commonly used in this step.
	
	A simple alternative is using the weighted mean aggregator, in which each agent averages its neighbors' models using a doubly stochastic mixing matrix $E \in \R^{W \times W}$. With the weighted mean aggregator, \eqref{eq: RAgg} becomes
	\begin{align}\label{eq: WeiMean}
		\hat{x}_w^{k+1} = & \text{WeiMean} \left( \{\hat{x}_v^{k+\frac{1}{2}}: v \in {\bar{\mathcal{N}}_w}\}\right) \\
		= & \sum_{v \in {\bar{\mathcal{N}}_w}} E_{wv} \hat{x}_v^{k+\frac{1}{2}}. \notag
	\end{align}
	The mixing matrix $E$ satisfies $E\bm{1}_W = \bm{1}_W$ and $\bm{1}_W^{\top}E = \bm{1}_W^{\top}$ with $\bm{1}_W \in \R^W$ being the $W$-dimensional all-one vector, and $E_{wv} \neq 0$ if and only if $(w ,v) \in \mathcal{E}$.

	We outline the decentralized gradient descent algorithm with  the robust aggregator $\text{RAgg}(\cdot)$ or the weighted mean aggregator $\text{WeiMean}(\cdot)$ in Algorithm 1.
	
	\begin{algorithm}[tb]
		\caption{Decentralized Gradient Descent with Robust and Weighted Mean Aggregators}
		\label{algorithm:1}
		\textbf{Input}:  Initialization $\hat{x}_w^0 = x^0 \in \R^D$ for all $w\in \mathcal{W}$; step size $\{\gamma^k, k = 0, 1, \cdots\}$; number of overall iterations $K$
		\begin{algorithmic}[1]
			\For{$k = 0, 1, \dots, K-1$}
			\State  Each agent $w$ computes the gradient $\nabla \hat{f}_w(\hat{x}_w^k)$ using
			\State  \quad its local data, and updates its local model via \eqref{eq: local update}.
			\State  Each agent $w$ sends $\hat{x}_w^{k+\frac{1}{2}}$ to its neighbors, and re-
			\State  \quad ceives $\{\hat{x}_v^{k+\frac{1}{2}} : v \in \mathcal{N}_w\}$.
			\State  Each agent $w$ updates its model $\hat{x}_w^{k+1}$ via \eqref{eq: RAgg} or \eqref{eq: WeiMean}.
			\EndFor
		\end{algorithmic}
	\end{algorithm}

	\section{Convergence Analysis}
	\label{sec: Convergence Analysis}
	
	Next, we analyze the convergence of Algorithm 1 with the robust and the weighted mean aggregators under label poisoning attacks. We first make the following assumptions. For regular agent $w \in \mathcal{R}$ (or poisoned agent $w \in \mathcal{W} \setminus \mathcal{R}$), we denote $x_w$ and $\nabla f_w$ (or $\tilde{x}_w$ and $\nabla \tilde{f}_w$) as its local model and local gradient, respectively.
	
	\begin{assumption}[\textbf{Lower boundedness}]\label{assump: lower boundedness}
		The global cost $f(\cdot)$ is lower bounded by $f^*$, i.e., $f(x) \geq f^*$ for any $x \in \R^D$.
	\end{assumption}
	
	\begin{assumption}[\textbf{Lipschitz continuous gradients}]\label{assump: L-smooth}
		The local cost $f_w(\cdot)$ of regular agent $w \in \mathcal{R}$ or $\tilde{f}_w(\cdot)$ of poisoned agent $w \in \mathcal{W} \setminus \mathcal{R}$ has $L$-Lipschitz continuous gradients, i.e., for any $x, y \in \R^D$, it holds that
		\begin{align}
			\|\nabla f_w(x) - \nabla f_w(y)\| &\leq L\|x - y\|, \ \forall w \in \mathcal{R}, \\
			\|\nabla \tilde{f}_w(x) - \nabla \tilde{f}_w(y)\| &\leq L\|x - y\|, \ \forall w \in \mathcal{W} \setminus \mathcal{R}.
		\end{align}
	\end{assumption}
	
	\begin{assumption}[\textbf{Bounded heterogeneity}]\label{assump: bounded heterogeneity}
		The maximal distance between the local gradients of regular agents $w \in \mathcal{R}$ and the global gradient is upper bounded by $\xi$, i.e., for any $x \in \R^D$, it holds that
		\begin{align}
			\max_{w \in \mathcal{R}} \|\nabla f_w(x) - \nabla f(x)\| \leq \xi.
		\end{align}
	\end{assumption}
	
	\begin{assumption}[\textbf{Bounded disturbances of poisoned local gradients}] \label{assump: bounded disturbances}
		The maximal distance between the poisoned local gradients of poisoned agents $w \in \mathcal{W} \setminus \mathcal{R}$ and the global gradient is upper bounded by $A$, i.e., for any $x \in \R^D$, it holds that
		\begin{align}\label{eq: bounded disturances}
			\max_{w \in \mathcal{W} \setminus \mathcal{R}} \|\nabla \tilde{f}_w(x) - \nabla f(x)\| \leq A.
		\end{align}
	\end{assumption}

	Assumptions \ref{assump: lower boundedness} and \ref{assump: bounded heterogeneity} are both standard and widely used in the analysis of decentralized first-order methods \cite{lian2017can, farhadkhani2023robust,wu2023byzantine}. Assumptions \ref{assump: L-smooth} and \ref{assump: bounded disturbances} involve poisoned agents, and do not hold under general model and data poisoning attacks, since the poisoned agents in such settings can arbitrarily manipulate their local models or data, leading to potentially non-Lipschitz local gradients and unbounded deviations from the global gradient. However, these assumptions can hold under weaker attacks, such as label poisoning that we are interested in. For example, as proved in \cite{peng2025mean}, for softmax regression, local labels do not affect the Hessian matrix of $\hat{f}_w$, and thus Assumption \ref{assump: L-smooth} is satisfied. Further, it is also proved that the local gradient $\nabla \hat{f}_w$ is bounded and Assumption \ref{assump: bounded disturbances} is satisfied.
	
	With particular note, \cite{peng2025mean} demonstrate that for softmax regression on non-i.i.d. data, the constant $A$ is of the same order as the constant $\xi$, i.e., $A = \Theta(\xi)$. The same conclusion holds for deep learning, meaning that $A$ and $\xi$ can be of the same order given sufficiently heterogeneous data distributions. We will use this observation in our subsequent analysis.

	\subsection{Main Results}
	To establish the convergence analysis of Algorithm \ref{algorithm:1} with robust aggregators, we need to formally characterize the robust aggregators. As discussed in \cite{vlaski2022robust}, the output of a robust aggregator should approximate a convex combination of the regular inputs. Based on this consideration, we define the robust aggregators as follows.

	\begin{definition}[\textbf{$(\rho, M)$-robust aggregator}]\label{def: rho-robust aggregator}
		Consider a virtual mixing matrix $M \! \in\!  \mathbb{R}^{R \times R}$ defined over the regular network, which consists of the regular agents and the edges connecting them.
		The $(w, v)$-th entry of $M$ satisfies $M_{wv} \in [0, 1]$ if $v \in {\bar{\mathcal{R}}_w}$, while $M_{wv} = 0$ otherwise. Further, $\sum_{v \in {\bar{\mathcal{R}}_w}} M_{wv} = 1$ for all $w \in \mathcal{R}$. For any regular agent $w \in \mathcal{R}$, consider any ${\bar{N}_w}$ vectors $\{y_v: v \in {\bar{\mathcal{N}}_w}\}$ among which {$\bar{R}_w$} vectors are regular. An aggregator $\text{RAgg}(\cdot)$ is said to be a $(\rho, M)$-robust aggregator if there exists a contraction constant $\rho \geq 0$ such that
		\begin{align}\label{eq: rho-robust aggregartor}
			& \left\|{\rm{RAgg}} \left(\{y_v:v \in {\bar{\mathcal{N}}_w}\}\right) - \bar{y}_w \right\| \\ \leq& \rho \cdot \max_{v \in {\bar{\mathcal{R}}_w}} \|y_v - \bar{y}_w\|, \nonumber
		\end{align}
		where $\bar{y}_w = \sum_{v \in {\bar{\mathcal{R}}_w}} M_{wv} y_v$ is a convex combination of the regular inputs.
	\end{definition}

	Definition \ref{def: rho-robust aggregator} characterizes the ``approximation abilities'' of the robust aggregators to a convex combination of regular inputs, and aligns with the definitions in {\cite{wu2023byzantine, ye2023tradeoff, yang2024byzantine, wang2024dual}. Many state-of-the-art robust aggregators, including but not limited to TriMean, FABA, CC, CG and IOS, are $(\rho, M)$-robust aggregators. Defining the \textbf{local contamination rate} $ \delta_{\max} \triangleq$ $\max_{w \in \mathcal{R}} (1 - {\frac{\bar{R}_w}{\bar{N}_w}}) $, which represents the maximal fraction of poisoned neighbors for the regular agents, the contraction constants $ \rho $ and the stochasticity of virtual mixing matrices $ M $ for these robust aggregators are presented in Table \ref{table: rho}.

	\begin{table}[!h]
		\centering
		\small 
		\setlength{\tabcolsep}{6pt} 
		\begin{tabular}{c c c c}
			\toprule
			& Aggregator & $\rho$ & $M$ \\
			\midrule
			& TriMean & $O\left(\frac{\delta_{\max}}{1 - 2\delta_{\max}}\right)$ & row-stochastic \\
			& FABA & $O\left(\frac{\delta_{\max}}{1 - 3\delta_{\max}}\right)$ & row-stochastic \\
			& CC & $O\left( \sqrt{\delta_{\max}} \right)$ & row-stochastic \\
			& CG & $O\left( \sqrt{\delta_{\max} (1 - \delta_{\max})} \right)$ & doubly stochastic \\
			& IOS & $O\left(\frac{\delta_{\max}}{1 - 3\delta_{\max}}\right)$ & doubly stochastic \\
			\bottomrule
		\end{tabular}
		\vspace{1.5mm}
		\caption{The contraction constants $ \rho $ and the stochasticity of virtual mixing matrices $ M $ of TriMean, FABA, CC, CG and IOS. The proofs for TriMean, FABA and CC are given in \cite{peng2025mean}, while the proofs for CG and IOS are given in \cite{wu2023byzantine, ye2023tradeoff}. }
		\label{table: rho}
	\end{table}

	
	

	Applying the contraction property in Definition \ref{def: rho-robust aggregator}, below we establish the convergence of Algorithm \ref{algorithm:1} with a  $(\rho, M)$-robust aggregator.
	
	\begin{theorem}\label{thm: RAgg}
		Consider Algorithm \ref{algorithm:1} with a  $(\rho, M)$-robust aggregator $\text{RAgg}(\cdot)$ to solve \eqref{problem: 1} under label poisoning attacks. Denote $p \in \R^R$ as the Perron vector of the virtual mixing matrix $M$ (i.e., $p^{\top}M = p^{\top}$ and $p^{\top} \bm{1}_R = 1$). Also define $\lambda \triangleq \|M - \bm{1}_{R} p^{\top}\|_{\infty}$ and $\beta \triangleq \|\bm{1}_{R}p^{\top} - \frac{1}{R} \bm{1}_{R}\bm{1}_{R}^{\top}\|_{\infty}$. Let $\gamma^k = \gamma = \frac{1}{\sqrt{K}}$, if $\lambda < 1$, $\rho < \frac{1 - \lambda}{4}$, and Assumptions \ref{assump: lower boundedness}, \ref{assump: L-smooth} and \ref{assump: bounded heterogeneity} are satisfied, when $K \geq \max\{\frac{16(\lambda + 4\rho)^2L^2}{(1 - \lambda-4\rho)^2}, L^2\}$, the consensus error of regular agents satisfies
		\begin{align}
			\max_{w \in \mathcal{R}} \|x_w^k - \bar{x}^k\|^2  = \ O\left(\frac{(\lambda + \rho)^2\xi^2}{K(1 - \lambda - 4\rho)^2}\right).
		\end{align}
		In addition, the average model of regular agents satisfies
		\begin{align}\label{eq: convergence for RAgg}
			&\frac{1}{K} \sum_{k=0}^{K-1} \|\nabla f(\bar{x}^k)\|^2 = \ O \left(\frac{f(x^0) - f^*}{\sqrt{K}}\right) \\ & + O\left(\frac{(\rho+1)^2L^2   (\lambda+\rho)^2\xi^2}{K(1-\lambda-4\rho)^2}\right) \nonumber \\&  + O\left(\left(\frac{\rho^2}{(1-\lambda-4\rho)^2} + \beta^2 \right)\xi^2\right). \nonumber
		\end{align}
		Here, $\bar{x}^k \triangleq \frac{1}{R} \sum_{w \in \mathcal{R}} x_w^k$.
	\end{theorem}
	
	From Theorem \ref{thm: RAgg}, we conclude that when the number of overall iterations $K$ is sufficiently large, the local models of regular agents will reach consensus, and the learning error of Algorithm \ref{algorithm:1} with a $(\rho, M)$-robust aggregator is in the order of $O((\frac{\rho^2 }{(1-\lambda-4\rho)^2} + \beta^2)\xi^2)$. When the virtual mixing matrix $M$ is doubly stochastic, which holds for CG and IOS as shown in Table \ref{table: rho}, we have $\beta = 0$ and the learning error becomes $O(\frac{\rho^2 \xi^2}{(1 - \lambda - 4\rho)^2})$. To be specific, when $\delta_{\max}$ and $\lambda$ are sufficiently small, the learning errors of CG and IOS are $O(\delta_{\max} \xi^2)$ and $O(\delta_{\max}^2 \xi^2)$, respectively. Further, when there are no poisoned agents, $\delta_{\max}=0$ so that $\rho = 0$, and the learning error becomes zero, aligning with our intuition.
	
	We are also able to prove the convergence of Algorithm \ref{algorithm:1} with the weighted mean aggregator under label poisoning attacks. Define the \textbf{global contamination rate} $ \delta \triangleq 1 - \frac{R}{W} $, which represents the fraction of poisoned agents for the entire network. We have the following result.

	\begin{theorem}\label{thm: WeiMean}
		Consider Algorithm \ref{algorithm:1} with the weighted mean aggregator $\text{WeiMean}(\cdot)$ to solve \eqref{problem: 1} under label poisoning attacks. Denote $\lambda' = \|E - \frac{1}{W} \bm{1}_W\bm{1}_W^{\top}\|$. Let $\gamma^k = \gamma = \frac{1}{\sqrt{K}}$, if Assumptions \ref{assump: lower boundedness}, \ref{assump: L-smooth}, \ref{assump: bounded heterogeneity} and \ref{assump: bounded disturbances} are satisfied, when $K \geq \max\{\frac{24\lambda'^2
			L^2}{(1 - \lambda')^2}, L^2\}$, the
		consensus error of regular agents satisfies
		\begin{align}
			& \max_{w \in \mathcal{R}} \|x_w^k - \bar{x}^k\|^2 \\
			& = O\left(\frac{(\lambda')^2(\delta^2 A^2 + \max\{\xi^2, A^2\})}{K(1 - \lambda')^2}\right). \notag
		\end{align}
		In addition, the average model of regular agents satisfies
		\begin{align}
			&\frac{1}{K} \sum_{k=0}^{K- 1} \|\nabla f(\bar{x}^k)\|  = O\left(\frac{f(x^0) - f^*}{\sqrt{K}}\right) \label{eq: convergence for WeiMean}\\& + O\left(\frac{L^2(\lambda')^2(\delta^2 A^2 + \max\{\xi^2, A^2\})}{K(1 - \lambda')^2}\right) + O(\delta^2A^2). \nonumber
		\end{align}
		Here, $\bar{x}^k\triangleq \frac{1}{R} \sum_{w \in \mathcal{R}} x_w^k$.
	\end{theorem}
	
	From Theorem \ref{thm: WeiMean}, we know that when the number of overall iterations $K$ is sufficiently large, the local models of regular agents will reach consensus, and the learning error of Algorithm \ref{algorithm:1} with the weighted mean aggregator $\text{WeiMean}(\cdot)$ is $O(\delta^2A^2)$ and topology-independent. When there are no poisoned agents, we have $\delta = 0$ and $A = 0$, and our results in Theorem \ref{thm: WeiMean} can recover the ones for attack-free decentralized optimization \cite{yuan2020influence}.

	
	\begin{remark}
		Note that our results in Theorems \ref{thm: RAgg} and \ref{thm: WeiMean} recover those in centralized learning \cite{peng2025mean}. In centralized learning, the central authority is in charge of the aggregation steps, equivalent to that the regular network is fully connected. In this case, $M = \frac{1}{R} \bm{1}_R \bm{1}_R^{\top}$ for all $(\rho, M)$-robust aggregators, and hence $\lambda = \beta = 0$. Therefore, the learning error in Theorem \ref{thm: RAgg} becomes $O(\frac{\rho^2 \xi^2}{(1 - 4\rho)^2})$, which recovers the learning error $O(\rho^2 \xi^2)$ in centralized learning when $\rho$ is small. For Theorem \ref{thm: WeiMean}, the learning error $O(\delta^2 A^2)$ exactly matches the result in centralized learning.
	\end{remark}
	
	
	\subsection{Lower Bounds of Learning Errors}

	To precisely compare the learning errors in Theorems \ref{thm: RAgg} and \ref{thm: WeiMean}, we should first establish corresponding lower bounds to demonstrate their tightness. To establish a meaningful lower bound for $(\rho, M)$-robust aggregator, we require it to be further \textit{majority-dominant}, formally defined as follows.

		\begin{definition}[\textbf{Majority-dominant aggregator}]\label{def: majority-dominant}
			An aggregator $\text{RAgg}(\cdot)$ is said to be a majority-dominant aggregator, if, for any agent $w\in\cal{W}$, whenever all regular inputs share the same vector $z$ and form a strict majority of its inputs, i.e., 
			$y_v = z, \forall v \in \bar{\mathcal{R}}_w \cap \mathcal{R}$
			and $|\bar{\mathcal{R}}_w \cap \mathcal{R}| > \frac{\bar{N}_w}{2}$, the aggregator outputs
			\begin{align}\label{eq: majority-dominant}
				{\rm{RAgg}}(\{y_v: v \in \bar{\mathcal{N}}_w\}) = z
			\end{align}
		\end{definition}
		
		Definition \ref{def: majority-dominant} requires the robust aggregator to output the same vector whenever the regular inputs are identical to one vector and form a majority. This definition is consistent with the ones given in \cite{guerraoui2024byzantine, zheng2025can}. Note that the majority-dominance property \eqref{eq: majority-dominant} is required to hold for any agent $w \in \mathcal{W}$, including both regular and poisoned agents, which is reasonable since poisoned agents also follow the algorithmic protocol. In the supplementary material, we show that TriMean, FABA, and IOS are majority-dominant, whereas CC and CG are not, as they fail to mimic the majority-vote scheme when the regular inputs are unanimous and form a majority.
		Combining with Table \ref{table: rho}, we conclude that TriMean, FABA and IOS are majority-dominant $(\rho, M)$-robust aggregators.

	Leveraging the majority-dominance property in Definition \ref{def: majority-dominant}, we provide the following lower bound for the learning error of Algorithm \ref{algorithm:1} with a {majority-dominant} $(\rho, M)$-robust aggregator under label poisoning attacks. The lower bound for the weighted mean aggregator is also provided as follows.
	
	\begin{theorem}\label{thm: lower bound}
		Consider Algorithm \ref{algorithm:1} running for $K$ iterations to solve \eqref{problem: 1} under label poisoning attacks. Given a {majority-dominant} $(\rho, M)$-robust aggregator, there exist $R$ regular local functions $\{f_w(x): w\in \mathcal{R}\}$ and $W - R$ poisoned local functions $\{\tilde{f}_w(x): w \in \mathcal{W} \setminus \mathcal{R}\}$ satisfying Assumptions \ref{assump: lower boundedness}, \ref{assump: L-smooth}, \ref{assump: bounded heterogeneity} and \ref{assump: bounded disturbances}, and a network topology in which {each agent $w \in \mathcal{W}$} has $N_w$ neighbor agents of which $R_w$ are regular, such that 
		\begin{align}\label{eq: lower bound for RAgg}
			\frac{1}{K} \sum_{k=1}^{K} \|\nabla f(\bar{x}^k)\|^2 = \Omega(\delta^2_{\max} \min\{A^2, \xi^2\}).
		\end{align}
		With the weighted mean aggregator, there also exist $R$ regular local functions $\{f_w(x): w\in \mathcal{R}\}$ and $W - R$ poisoned local functions $\{\tilde{f}_w(x): w \in \mathcal{W} \setminus \mathcal{R}\}$ satisfying Assumptions \ref{assump: lower boundedness}, \ref{assump: L-smooth}, \ref{assump: bounded heterogeneity} and \ref{assump: bounded disturbances}, and a network topology in which {each agent $w \in \mathcal{W}$} has $N_w$ neighbor agents of which $R_w$ are regular, such that 
		\begin{align}\label{eq: lower bound for WeiMean}
			\frac{1}{K} \sum_{k=1}^{K} \|\nabla f(\bar{x}^k)\|^2 = \Omega(\delta^2 \min\{A^2, \xi^2\}).
		\end{align}
	\end{theorem}
	
	Below we compare the lower bounds in Theorem \ref{thm: lower bound} and the upper bounds in Theorems \ref{thm: RAgg} and \ref{thm: WeiMean}. We focus on the case that the data distribution is sufficiently heterogeneous such that $A = \Theta(\xi)$; as discussed in \cite{peng2025mean}. In this case, the lower bound in \eqref{eq: lower bound for RAgg} is $\Omega(\delta_{\max}^2 \xi^2)$. According to the discussion below Theorem \ref{thm: RAgg}, we find that the learning error of IOS exactly matches the lower bound. In contrast, the learning errors of {TriMean and FABA} do not match the lower bound due to the non-doubly stochastic nature of their virtual mixing matrices $M$.
	Further, when $A = \Theta(\xi)$, the upper bound of weighted mean aggregator in Theorem \ref{thm: WeiMean} is $O(\delta^2 \xi^2)$, matching the lower bound $\Omega(\delta^2 \xi^2)$ in \eqref{eq: lower bound for WeiMean}. Therefore, the bounds for {majority-dominant} $(\rho, M)$-robust and weighted mean aggregators are all tight.
	
	We summarize these tight bounds in Table \ref{table: learning error}. We observe that the weighted mean aggregator outperforms the majority-dominant $(\rho, M)$-robust aggregators when \textbf{the global contamination rate is smaller than the local contamination rate} (i.e., $\delta < \delta_{\max}$), given that the data heterogeneity is sufficiently high (i.e., $A = \Theta(\xi)$).
	
	
	\begin{table}[!h]
		\centering
		\begin{tabular}{ccc}
			\toprule
			&Aggregator  & Learning error \\
			\midrule
			&RAgg   &  $\Theta(\delta_{\max}^2\xi^2)$  \vspace{1.5mm} \\
			&WeiMean & $\Theta(\delta^2 \xi^2)$  \\
			\bottomrule
		\end{tabular}
		\vspace{1.5mm}
		\caption{Learning errors of Algorithm \ref{algorithm:1} with the optimal {majority-dominant} $(\rho, M)$-robust aggregator $\text{RAgg}(\cdot)$ and the weighted mean aggregator $\text{WeiMean}(\cdot)$, given large heterogeneity such that $A = \Theta(\xi)$.
		}
		\label{table: learning error}
	\end{table}
	
	Looking deeper into the upper bounds, we can find more advantages of the weighted mean. Note that the upper bound of  $(\rho, M)$-robust aggregators relies on the conditions $\lambda < 1$ and $\rho < \frac{1 - \lambda}{4}$ in Theorem \ref{thm: RAgg}. Here, $\lambda$ characterizes the sparsity level of the regular network; smaller $\lambda$ indicates better connectivity. Thus, $\lambda < 1$ requires the regular network to be well-connected. Further, $\rho < \frac{1 - \lambda}{4}$ requires the contraction constant $\rho$ to be sufficiently small (and thus the local contamination rate $\delta_{\max}$ is small, according to Table \ref{table: rho}), even when the regular network is well-connected. These conditions do not necessarily hold. As a matter of fact, all $(\rho, M)$-robust aggregators mentioned in this paper have $\lambda > 1$ on the fan graph. In contrast, the weighted mean aggregator does not require such conditions. This observation suggests that the weighted mean aggregator may outperform the $(\rho, M)$-robust aggregators when these conditions are not satisfied, when for example, \textbf{the regular network is disconnected}, or \textbf{the regular network is sparse and the local contamination rate is large}.
	
%

	\begin{remark}
			In the theoretical analysis, we consider deterministic optimization where each agent computes full local gradients. In the ensuing numerical experiments, we study both deterministic and stochastic optimization, with agents computing either full or stochastic local gradients, and show that our results also hold empirically in the stochastic setting. The main challenge in analyzing stochastic optimization lies in deriving a tight upper bound, which typically requires variance reduction techniques to mitigate the effect of stochastic gradient noise. We will fill this gap in our future work.
		\end{remark}

	\section{Numerical Experiments}
	\label{sec: numerical experiments}
	
	In this section, we conduct numerical experiments to validate our theoretical results and demonstrate the performance of Algorithm \ref{algorithm:1} with robust and weighted mean aggregators. The experimental settings are as follows. The code is available at \url{https://github.com/pengj97/DLPA}.
	
	\textbf{Network topologies.} We consider three network topologies that correspond to the three scenarios that we are interested in: two-castle, line and fan, as shown in Figure \ref{fig:graph_topology} in the supplementary material. Each graph contains $W = 10$ agents, among which $R = 9$ are regular. We also vary the numbers of regular agents and all agents in the supplementary material.

	\textbf{Datasets and partitions.} We train a softmax regression model on the MNIST dataset and a ResNet18 model on the CIFAR100 dataset. We also train larger neural networks using larger datasets in the supplementary material. We consider three types of data partition across the agents: i.i.d., mild non-i.i.d. and non-i.i.d.. In the i.i.d. case, we evenly randomly distribute the training data to all agents. In the mild non-i.i.d. case, we divide the training data following the Dirichlet distribution with hyper-parameter $\alpha=1$ \cite{hsu2019measuring}. In the non-i.i.d. case, we assign ten unique classes to each agent.
	
	\textbf{Label poisoning attacks.} Following \cite{shejwalkar2022back,peng2025mean}, we consider the label flipping attacks, with which each poisoned agent flips its local sample labels from $b$ to $B - b$ where $B$ is the number of classes and label $b \in \{0, 1, \cdots, B - 1\}$. 
	
	\textbf{Aggregators.} We compare the weighted mean aggregator against several representative $(\rho, M)$-robust aggregators, including TriMean, FABA, CC, CG, IOS, {RFA}, and LFighter. The baseline is the weighted mean without any attacks and the mixing matrix is Metropolis-Hastings.

\textbf{Parameters.} While our theories are established for a small constant step size, we adopt a diminishing schedule in the numerical experiments to accelerate training, as the larger initial steps lead to faster progress in practice. We use $\gamma^k = \frac{0.1}{\sqrt{k}}$ and $\gamma^k = \frac{0.03}{\sqrt{k}}$ for softmax regression and neural network training, respectively. Besides, for softmax regression, we use full local gradients; for the more computation-intensive neural network training, we use stochastic local gradients with batch size being 32.

	\begin{figure}
		\centering
		\includegraphics[scale=0.105]{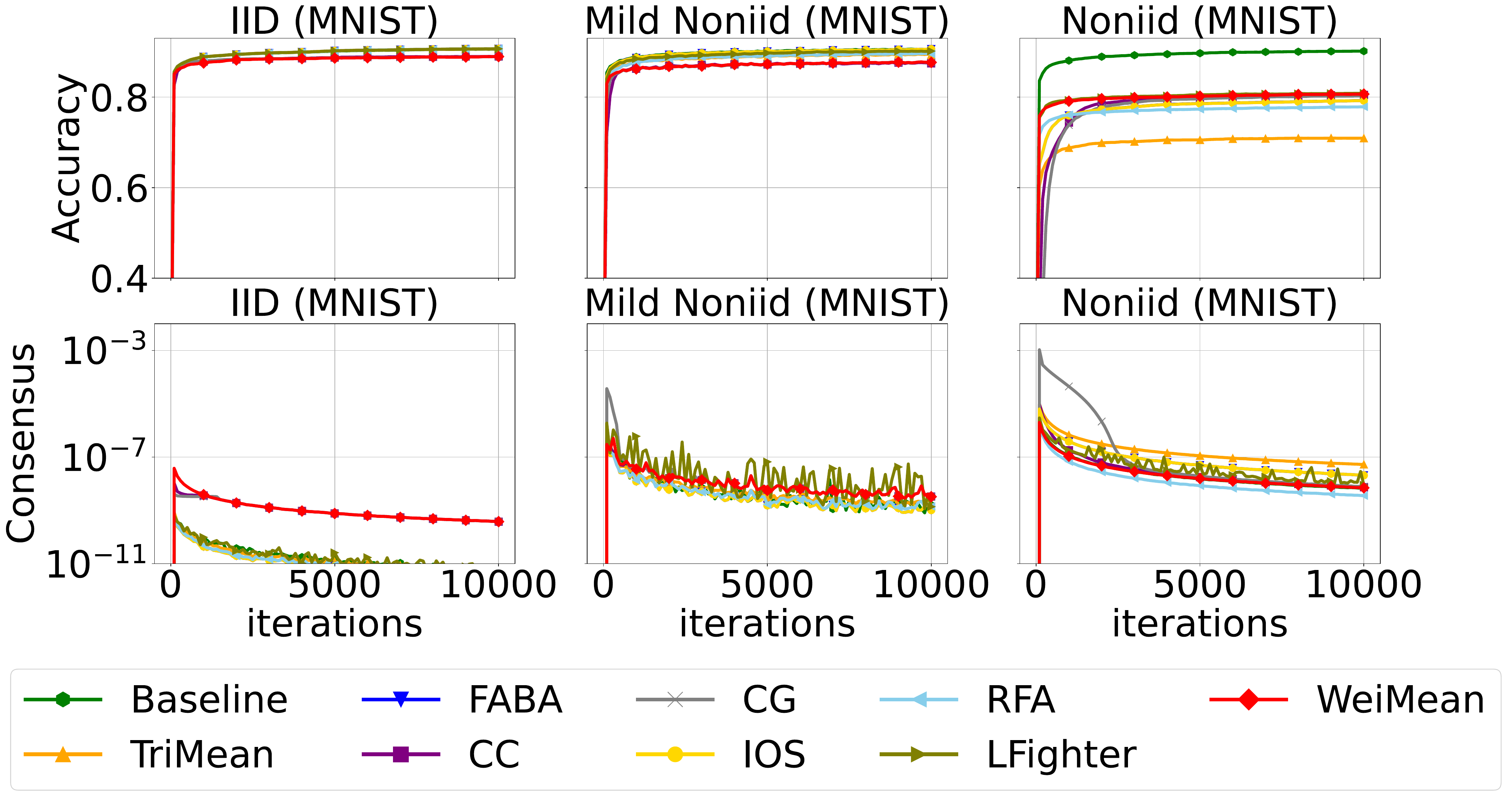}
		\caption{Classification accuracies and consensus errors of softmax regression model trained on MNIST in the two-castle graph.}
		\label{fig:SR_mnist_twocastle_label_flipping}
	\end{figure}
	
	\begin{figure}
		\centering
		\includegraphics[scale=0.105]{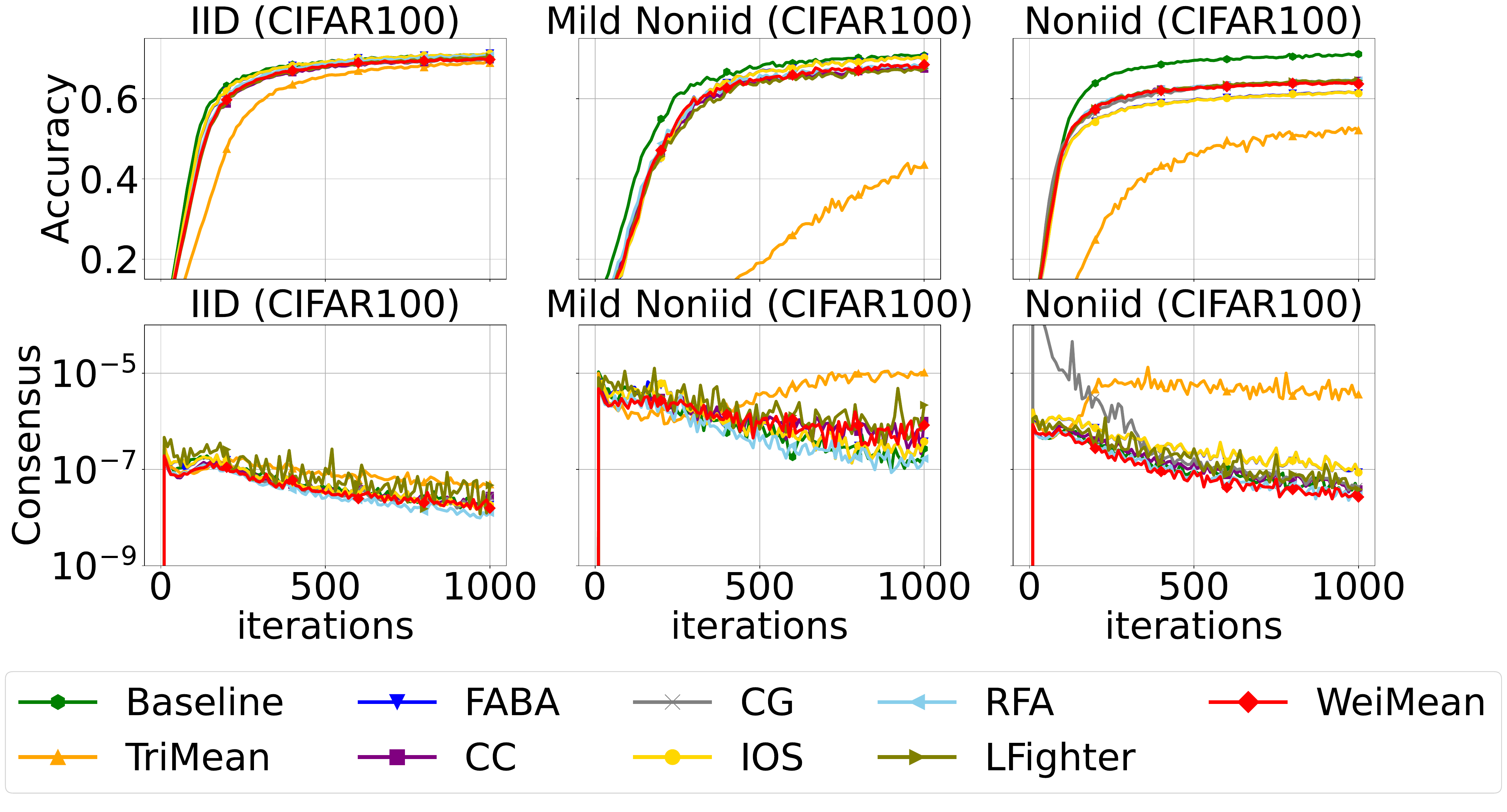}
		\caption{Classification accuracies and consensus errors of ResNet18 {trained} on CIFAR100 in the two-castle graph.}
		\label{fig:ResNet18_cifar100_twocastle_label_flipping}
	\end{figure}
	
		\begin{figure}
		\centering
		\includegraphics[scale=0.105]{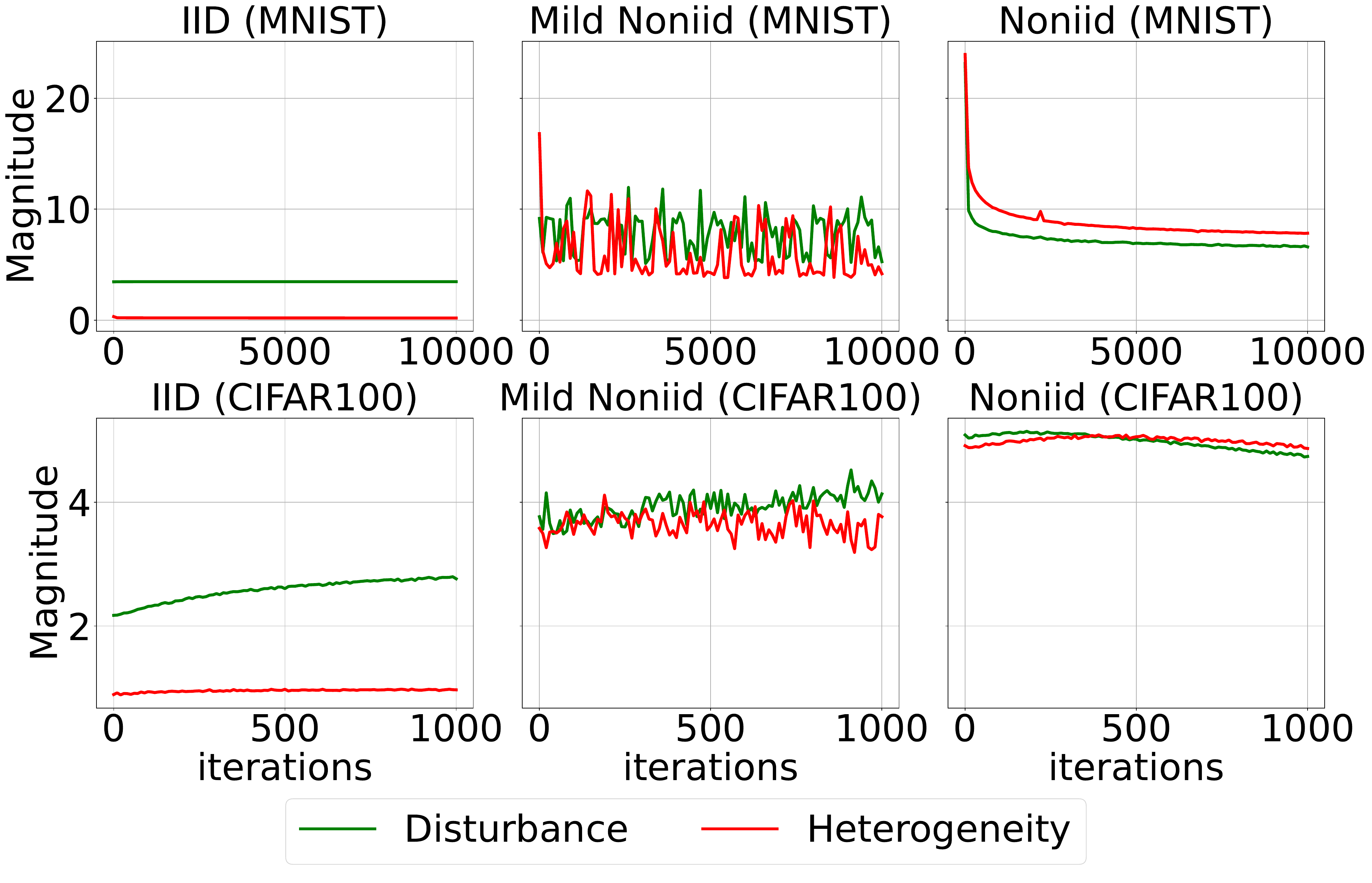}
		\caption{Heterogeneity of regular local gradients and disturbances of poisoned local gradients for softmax regression trained on MNIST and ResNet18 trained on CIFAR100 in the two-castle graph.}
		\label{fig:SR_MNIST_ResNet18_cifar100_twocastle_A_hetero}
	\end{figure}
	
	\textbf{Scenario (i): smaller global contamination rate.} In this experiment, we consider the two-castle graph, in which the global contamination rate $\delta = \frac{1}{10}$ is lower than the local contamination rate $\delta_{\max} = \frac{1}{9}$. The results are shown in Figures \ref{fig:SR_mnist_twocastle_label_flipping} and \ref{fig:ResNet18_cifar100_twocastle_label_flipping}. In the i.i.d. case, all methods achieve high classification accuracy, comparable to the baseline. In the mild non-i.i.d. case, FABA and IOS yield the highest accuracies among all aggregators.
	In the non-i.i.d. case, most aggregators suffer from performance degradation due to increased heterogeneity. However, the weighted mean aggregator maintains strong performance and ranks among the best, which aligns with our theoretical findings in Table \ref{table: learning error}. Regarding consensus, all aggregators achieve a final consensus error in the order of $10^{-5}$ across the i.i.d., mild non-i.i.d. and non-i.i.d. cases, indicating that regular models nearly reach consensus given sufficient iterations. This observation is consistent with our theoretical results in Theorems \ref{thm: RAgg} and \ref{thm: WeiMean}.

	To further investigate the results in Figures \ref{fig:SR_mnist_twocastle_label_flipping} and \ref{fig:ResNet18_cifar100_twocastle_label_flipping} and validate Assumptions \ref{assump: bounded heterogeneity} and \ref{assump: bounded disturbances}, we compute the heterogeneity of regular local gradients
	(i.e., $\max_{w \in \mathcal{R}}\|\nabla f_w(\bar{x}^k) -\nabla f(\bar{x}^k)\|$, $\forall k \in [K]$) and the disturbance of poisoned local gradients
	(i.e., $\max_{w \in \mathcal{W} \setminus \mathcal{R}}\|\nabla \tilde{f}_w(\bar{x}^k) - \nabla f(\bar{x}^k)\|$, $\forall k \in [K]$) on the two-castle graph. As shown in Figure \ref{fig:SR_MNIST_ResNet18_cifar100_twocastle_A_hetero}, both the heterogeneity $\xi$ and disturbance $A$ are bounded, which validate the practicality of Assumptions \ref{assump: bounded heterogeneity} and \ref{assump: bounded disturbances}. From i.i.d., mild non-i.i.d. to non-i.i.d. cases, the heterogeneity $\xi$ increases, and the heterogeneity $\xi$ is in the same order of the disturbance $A$ in the non-i.i.d. case. According to Table \ref{table: learning error}, when $\xi$ is in the same order of $A$ and the global contamination rate is smaller than the local contamination rate, the weighted mean aggregator outperforms the optimal majority-dominant $(\rho, M)$-robust aggregator, which explains the results in Figure \ref{fig:ResNet18_cifar100_twocastle_label_flipping}.

	\textbf{Scenario (ii): disconnected regular network.}
	In this experiment, we evaluate the line graph in which the regular network is disconnected. The results are depicted in Figures \ref{fig:SR_mnist_disconnected_label_flipping} and \ref{fig:ResNet18_cifar100_disconnected_label_flipping}.
	In the i.i.d. case, low heterogeneity across the regular local gradients enables each agent to learn effectively from its own data, and all aggregators achieve the accuracies close to the baseline. In the mild non-i.i.d. case, increased heterogeneity degrades every method. In the non-i.i.d. case, TriMean, FABA, IOS, RFA and LFighter suffer severe accuracy drops due to the combination of regular network disconnection and label flipping attacks, corroborating our theoretical findings. Recalled that according to Theorem \ref{thm: RAgg}, when $\lambda\ge1$, Algorithm \ref{algorithm:1} with any $(\rho,M)$-robust aggregator has no theoretical guarantee and may perform poorly. Actually, $\lambda\ge1$ holds for the line graph that we are investigating, explaining the observed performance degradation. In contrast, the weighted mean aggregator, whose learning error is topology-independent, remains the best in the non-i.i.d. case.

	\begin{figure}
		\centering	
		\includegraphics[scale=0.105]{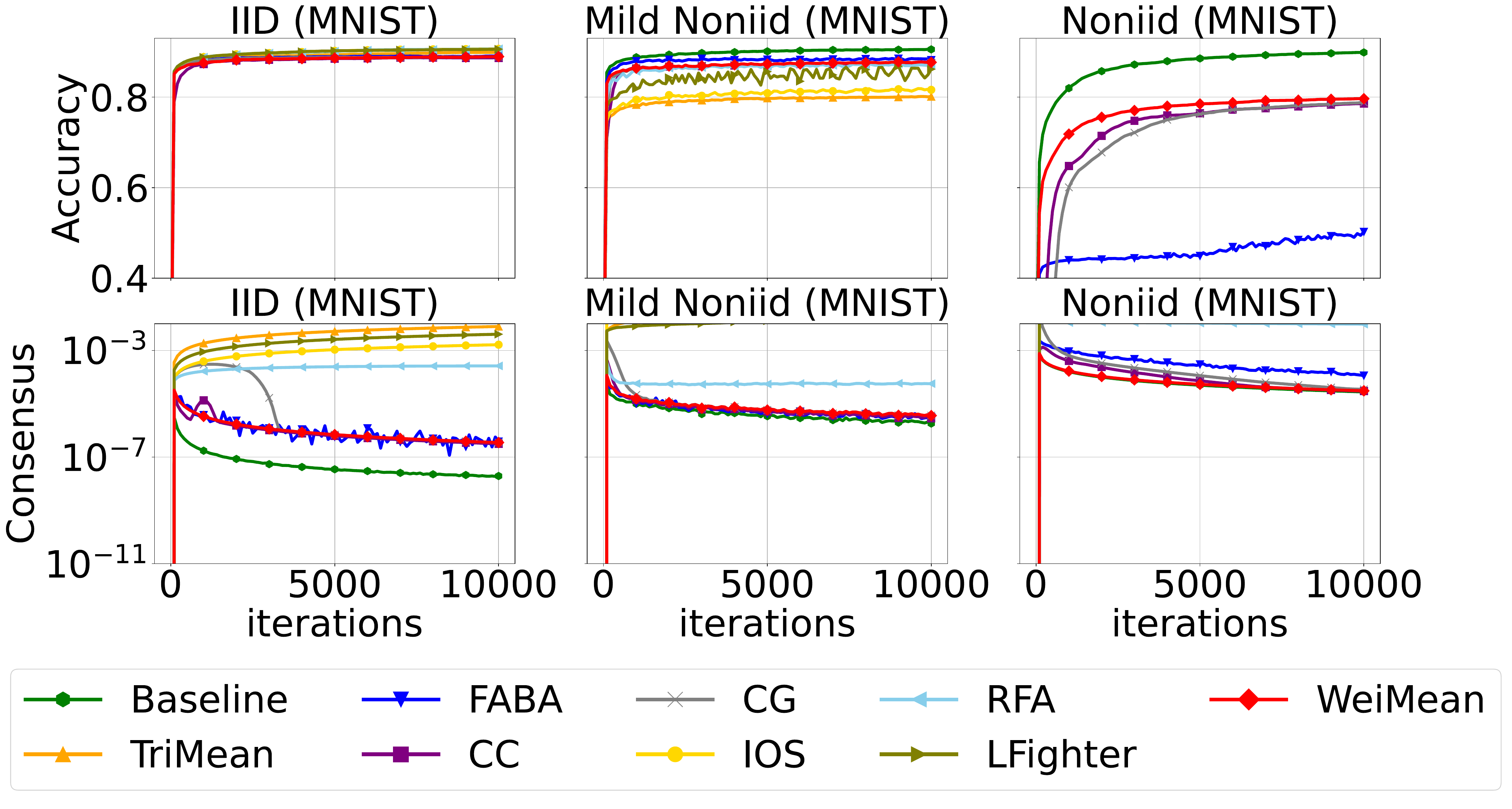}
		\caption{Classification accuracies and consensus errors of softmax regression trained on MNIST in the line graph.}
		\label{fig:SR_mnist_disconnected_label_flipping}
	\end{figure}
	
		\begin{figure}
		\centering
		\includegraphics[scale=0.105]{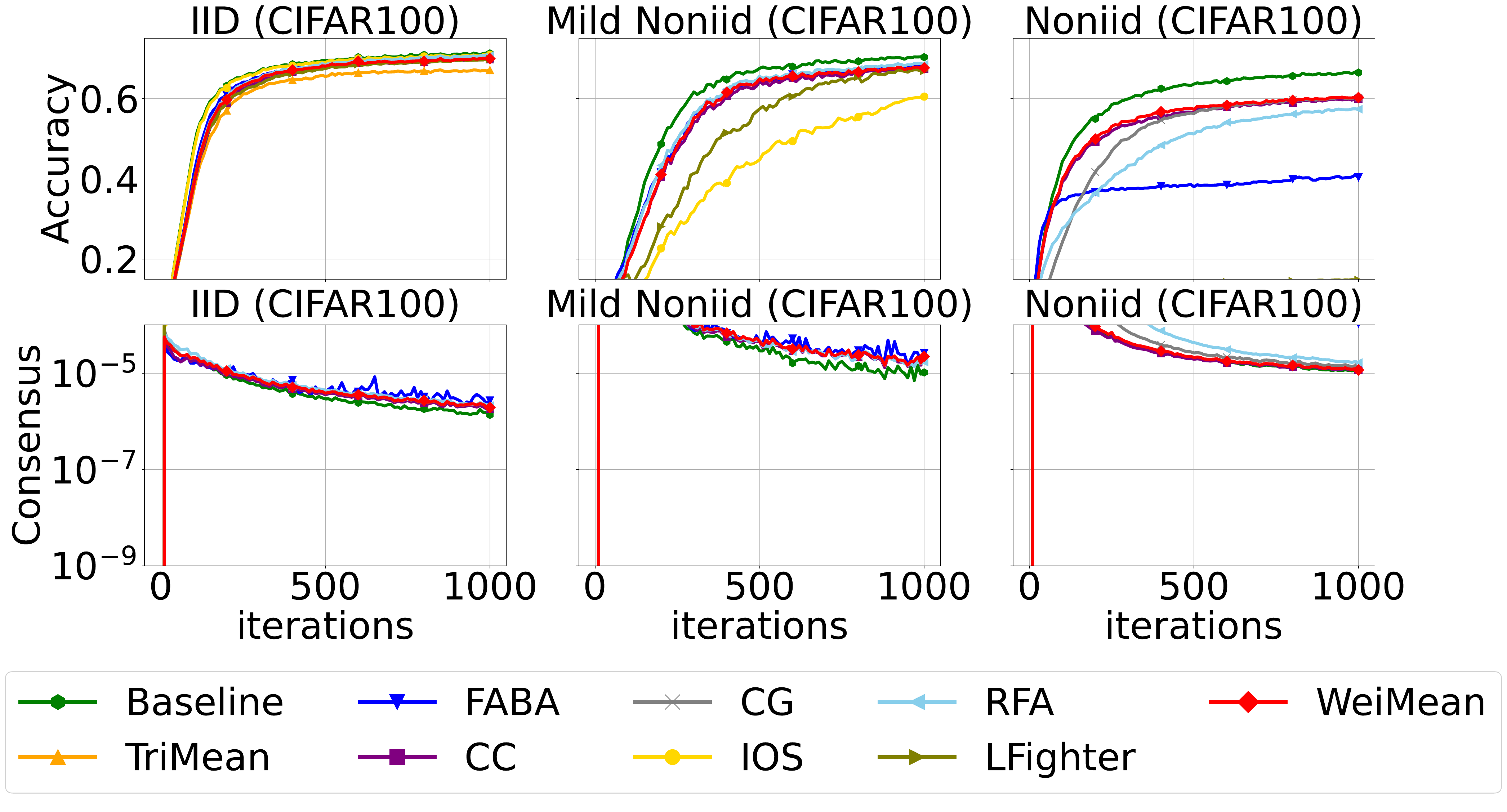}
		\caption{Classification accuracies and consensus errors of ResNet18 {trained} on CIFAR100 in the line graph.}
		\label{fig:ResNet18_cifar100_disconnected_label_flipping}
	\end{figure}

	\textbf{Scenario (iii): sparse regular network and high local contamination rate.} To illustrate the combined effect of a sparse regular network and high local contamination rate, we use the fan graph, where the regular network is a line graph and the local contamination rate $\delta_{\max} = \frac{1}{3}$. The results for CIFAR100 are shown in Figure As depicted in Figures \ref{fig:SR_mnist_fan_label_flipping} and \ref{fig:ResNet18_cifar100_fan_label_flipping}, in the i.i.d. case, minimal heterogeneity favors robust aggregators (e.g., IOS) over the weighted mean aggregator. In the mild non-i.i.d. case, as the heterogeneity increases, the performance of robust aggregators begins to degrade. In the non-i.i.d. case, sparse connectivity of regular network and high local contamination rate cause severe accuracy drops for TriMean, FABA and IOS.
	In contrast, the weighted mean consistently attains the best accuracy, further corroborating our theoretical results.

	\begin{figure}
		\centering
		\includegraphics[scale=0.105]{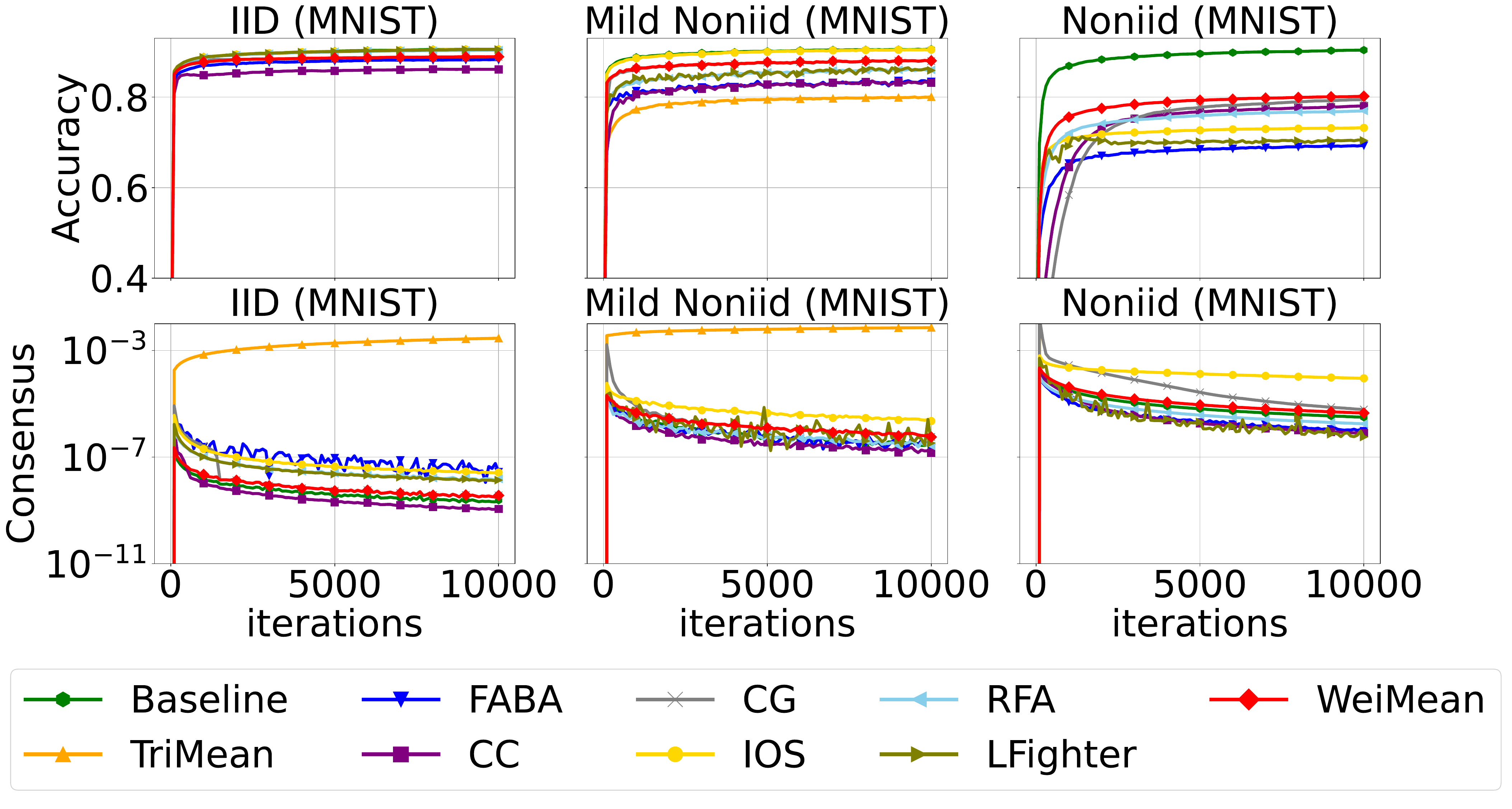}
		\caption{Classification accuracies and consensus errors of softmax regression trained on MNIST in the fan graph.}
		\label{fig:SR_mnist_fan_label_flipping}
	\end{figure}
	
		\begin{figure}
		\centering
		\includegraphics[scale=0.105]{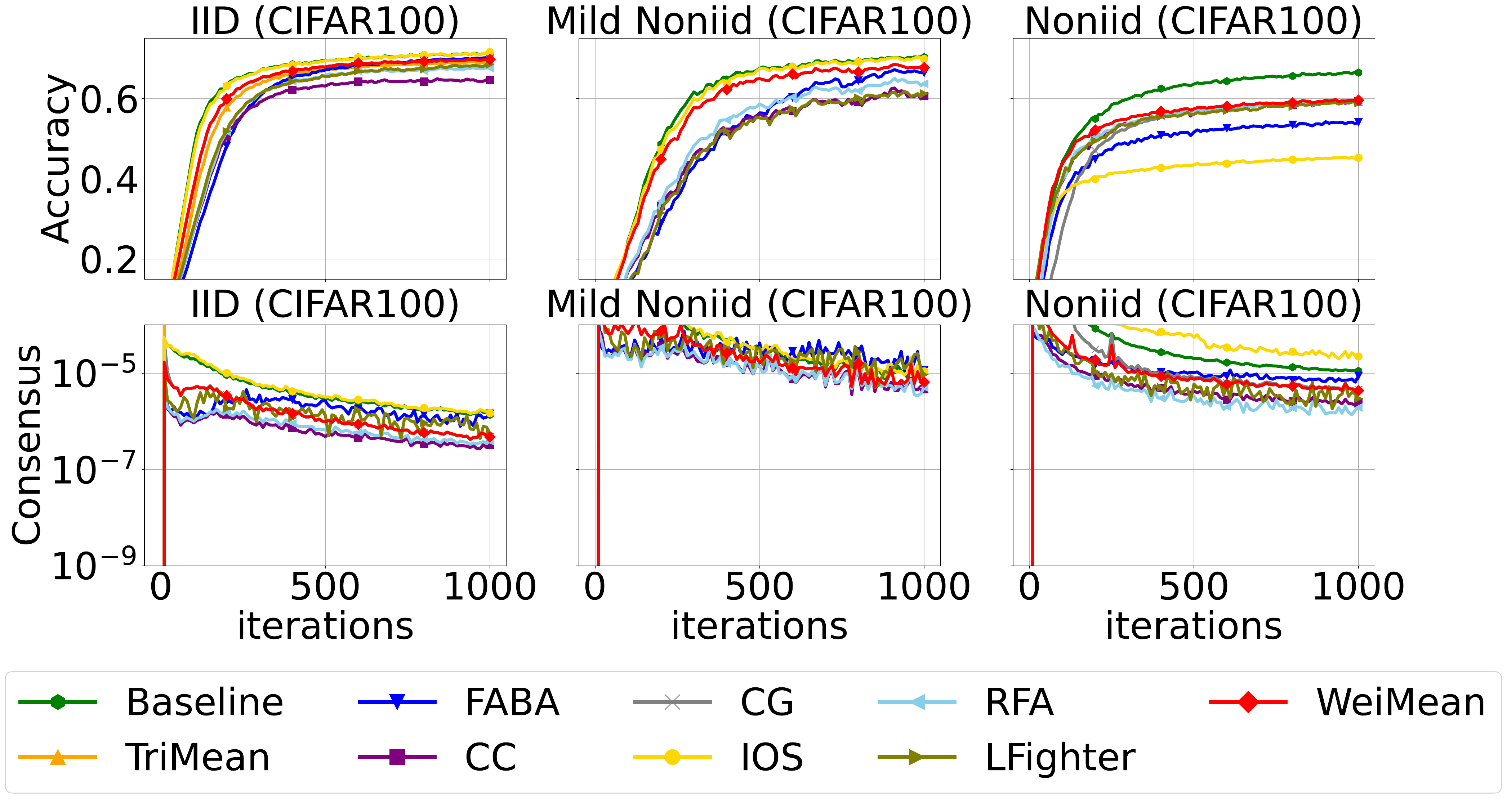}
		\caption{Classification accuracies and consensus errors of ResNet18 {trained} on CIFAR100 in the fan graph.}
		\label{fig:ResNet18_cifar100_fan_label_flipping}
	\end{figure}

	\section{Conclusion}
	In this paper, we consider label poisoning attacks, the weak yet practical attacks, where poisoned agents manipulate their local data at the label level and update their local models using these poisoned data. Under label poisoning attacks, we surprisingly find that the simple weighted mean aggregator can be more robust than the robust aggregators in highly heterogeneous cases, provided that: (i) the global contamination rate is smaller than the local contamination rate; (ii) the regular network is disconnected; or (iii) the regular network is sparse and the local contamination rate is high. Our findings highlight an important insight: robust aggregators do not always outperform the weighted mean aggregator under specific attacks and across all network topologies. This encourages us to reconsider the appropriate application scenarios for different aggregators. While our work focuses on the robustness, future research should also weigh computational complexity to identify aggregators that best balance the efficiency and the robustness.

	\appendices

	\section*{Proof of Theorem \ref{thm: WeiMean}}
	
	\begin{proof}
		For notational convenience, lete $\bar{x}_W = \frac{1}{W} \sum_{w \in \mathcal{W}} \hat{x}_w$ and denote three ${W \times D}$ matrices
		\begin{align*}
			\hspace{-1em} \hat{X} =  \begin{bmatrix} (\hat{x}_1)^{\top} \\ (\hat{x}_2)^{\top} \\ \vdots \\ (\hat{x}_W)^{\top}\end{bmatrix}\!, \bar{X}_W = \begin{bmatrix} (\bar{x}_W)^{\top} \\ (\bar{x}_W)^{\top}  \\ \vdots \\ (\bar{x}_W)^{\top} \end{bmatrix}\!, \nabla \mathbf{\hat{f}}(\hat{X}) = \begin{bmatrix} (\nabla \hat{f}_1(\hat{x}_1))^{\top} \\ \vdots \\ (\nabla \hat{f}_W(\hat{x}_W))^{\top}   \end{bmatrix}\!.
		\end{align*}
		
		With Assumption \ref{assump: L-smooth}, we have
		\begin{align}\label{proof:thm2-1}
			&f(\bar{x}_W^{k+1}) \\ \leq &f(\bar{x}_W^k) + \langle \nabla f(\bar{x}_W^k), \bar{x}_W^{k+1} - \bar{x}_W^k \rangle + \frac{L}{2} \|\bar{x}_W^{k+1} - \bar{x}_W^k\|^2. \nonumber
		\end{align}
		For the second term on the right-hand side of \eqref{proof:thm2-1}, we have
		\begin{align} \label{proof:thm2-2}
			&\langle \nabla f(\bar{x}_W^k), \bar{x}_W^{k+1} - \bar{x}_W^k \rangle \\ =& \gamma \langle \nabla f(\bar{x}_W^k), \frac{1}{\gamma}(\bar{x}_W^{k+1} - \bar{x}_W^k) \rangle \nonumber \\ =& \frac{\gamma}{2} \|\nabla f(\bar{x}_W^k) + \frac{1}{\gamma}(\bar{x}_W^{k+1} - \bar{x}_W^k)\|^2 - \frac{\gamma}{2} \|\nabla f(\bar{x}_W^k)\|^2 \nonumber \\&- \frac{\gamma}{2} \|\frac{1}{\gamma}(\bar{x}_W^{k+1} - \bar{x}_W^k)\|^2. \nonumber
		\end{align}
		Thus, with the step size $\gamma \leq \frac{1}{L}$, we have
		\begin{align}\label{proof:thm2-3}
			&f(\bar{x}_W^{k+1})\\
			\leq &f(\bar{x}_W^k) + \frac{\gamma}{2} \|\nabla f(\bar{x}_W^k) + \frac{1}{\gamma}(\bar{x}_W^{k+1} - \bar{x}_W^k)\|^2 - \frac{\gamma}{2} \|\nabla f(\bar{x}_W^k)\|^2 \nonumber\\+& (\frac{L\gamma^2}{2} - \frac{\gamma}{2}) \|\frac{1}{\gamma}(\bar{x}_W^{k+1} - \bar{x}_W^k)\|^2 \nonumber \\
			\leq &f(\bar{x}_W^k) + \frac{\gamma}{2} \|\nabla f(\bar{x}_W^k) + \frac{1}{\gamma}(\bar{x}_W^{k+1} - \bar{x}_W^k)\|^2 - \frac{\gamma}{2} \|\nabla f(\bar{x}_W^k)\|^2. \nonumber
		\end{align}
		
		Now we handle the second term on the right-hand side of \eqref{proof:thm2-3}. It holds that
		\begin{align}\label{proof:thm2-4}
			&\|\nabla f(\bar{x}_W^k) + \frac{1}{\gamma} (\bar{x}_W^{k+1} - \bar{x}_W^k)\|^2 \\
			=& \|\nabla f(\bar{x}_W^k) + \frac{1}{\gamma} (\frac1W\sum_{w \in \mathcal{W}}\sum_{v\in {\bar{\mathcal{N}}_w} }E_{wv}(\hat{x}_v^k - \gamma \cdot \nabla f_v(\hat{x}_v^k)) \nonumber \\& - \frac{1}{W} \sum_{w \in \mathcal{W}}\hat{x}_w^k)\|^2\nonumber \\ =& \|\nabla f(\bar{x}_W^k) -\frac1W\sum_{ w \in \mathcal{W}} \nabla \hat{f}_w(\hat{x}_w^k) \|^2\nonumber \\
			\leq& 2\|\frac{1}R \sum_{w \in \mathcal{R}} \nabla f_w(\bar{x}_W^k) - \frac{1}W \sum_{w \in \mathcal{W}} \nabla \hat{f}_w(\bar{x}_W^k)\|^2 \nonumber \\& +  2\|\frac{1}W \sum_{w \in \mathcal{W}} \nabla \hat{f}_w(\bar{x}_W^k) - \frac{1}W \sum_{w \in \mathcal{W}} \nabla \hat{f}_w(\hat{x}_w^k)\|^2 \nonumber \\
			\leq& 2\|\frac{1}R \sum_{w \in \mathcal{R}} \nabla f_w(\bar{x}_W^k) - \frac{1}W \sum_{w \in \mathcal{W}} \nabla \hat{f}_w(\bar{x}_W^k)\|^2 \nonumber\\& +  2 \frac{1}W \sum_{w \in \mathcal{W}}\| \nabla \hat{f}_w(\bar{x}_W^k) - \nabla \hat{f}_w(\hat{x}_w^k)\|^2 \nonumber \\
			\leq&2\delta^2 A^2 +\frac{2L^2}{W} \|\hat{X}^k - \bar{X}_W^k\|_F^2, \nonumber
		\end{align}
		where the second equality comes from the double stochasticity of the mixing matrix $E$, the second inequality is due to the Cauchy-Schwarz inequality and the last inequality comes from Assumptions \ref{assump: L-smooth} and \ref{assump: bounded disturbances}.
		Plugging \eqref{proof:thm2-4} back to \eqref{proof:thm2-3}, we have
		\begin{align} \label{proof:thm2-5}
			&\|\nabla f(\bar{x}_W^k)\|^2  \\\leq &\frac{2(f(\bar{x}_W^k) - f(\bar{x}_W^{k+1}))}{\gamma} + \frac{2L^2}{W} \|\hat{X}^k - \bar{X}_W^k\|_F^2 + 2\delta^2A^2. \nonumber
		\end{align}
		
		For \eqref{proof:thm2-5}, since
		\begin{align} \label{proof:thm2-6}
			&\|\nabla f(\bar{x}^k)\|^2 \\ \leq & 2\|\nabla f(\bar{x}_W^k)\|^2 + 2\|\nabla f(\bar{x}^k) - \nabla f(\bar{x}_W^k)\|^2 \nonumber \\ \leq&2\|\nabla f(\bar{x}_W^k)\|^2 +\frac{2}{R} \sum_{w\in\mathcal{R}} L^2\|x_w^k - \bar{x}_W^k\|^2 \nonumber\\ \leq & 2\|\nabla f(\bar{x}_W^k)\|^2 + \frac{2L^2}{R} \|\hat{X}^k - \bar{X}_W^k\|_F^2, \nonumber
		\end{align}
		where the second inequality comes from Assumption \ref{assump: L-smooth}, with the fact that $R \leq W$, we further have
		\begin{align} \label{proof:thm2-6-2}
			&\|\nabla f(\bar{x}^k)\|^2 \\ \le &\frac{4(f(\bar{x}_W^k) - f(\bar{x}_W^{k+1}))}{\gamma} + \frac{6L^2}{R} \|\hat{X}^k - \bar{X}_W^k\|_F^2 + 4\delta^2A^2. \nonumber
		\end{align}
		
		Now we know that $\|\nabla f(\bar{x}^k)\|$ is bounded by  $\|\hat{X}^k - \bar{X}_W^k\|_F^2$, in addition to other terms. Next, we recursively bound this term. Note that
		\begin{align} \label{proof:thm2-7}
			&\|\hat{X}^{k+1} - \bar{X}_W^{k+1} \|_F \\
			=& \|\hat{X}^{k+1} - \frac1W\bm{1}_W\bm{1}_W^{\top} \hat{X}^{k+1}\|_F \nonumber\\ =& \|E\hat{X}^{k+\frac12} - \frac1W\bm{1}_W\bm{1}_W^{\top} E\hat{X}^{k+\frac12}\|_F \nonumber\\ \leq& \|E - \frac{1}{W} \bm{1}_W\bm{1}_W^{\top}\| \|\hat{X}^{k+\frac12} - \frac1W\bm{1}_W\bm{1}_W^{\top} \hat{X}^{k+\frac12}\|_F \nonumber \\ =&\lambda' \cdot \|\hat{X}^{k+\frac12} - \frac1W\bm{1}_W\bm{1}_W^{\top} \hat{X}^{k+\frac12}\|_F, \nonumber
		\end{align}
		Let $\omega \triangleq 1 -  \lambda'$, since $0 \leq \lambda' <1$ (according to the Perron-Frobenius Theorem), we have $\omega \in (0, 1]$ and
		\begin{align} \label{proof:thm2-7-2}
			&\|\hat{X}^{k+1} - \bar{X}_W^{k+1} \|_F \\ \leq &(1 - \omega) \cdot \|\hat{X}^{k+\frac12} - \frac1W\bm{1}_W\bm{1}_W^{\top} \hat{X}^{k+\frac12}\|_F. \nonumber
		\end{align}
		For the term on the right-hand side of \eqref{proof:thm2-7-2}, we have
		\begin{align}\label{proof:thm2-8}
			&\|\hat{X}^{k+\frac12} - \frac1W\bm{1}_W\bm{1}_W^{\top} \hat{X}^{k+\frac12}\|_F \\=& \|\hat{X}^{k}  - \gamma \cdot \nabla\mathbf{\hat{f}}(\hat{X}^k)- \frac1W\bm{1}_W\bm{1}_W^{\top}(\hat{X}^{k}  - \gamma \cdot \nabla \mathbf{\hat{f}}(\hat{X}^k))\|_F \nonumber \\ =&\|\hat{X}^{k}  - \gamma \cdot \nabla\mathbf{\hat{f}}(\hat{X}^k)- (\bar{X}_W^{k}  - \gamma \cdot \frac1W\bm{1}_W\bm{1}_W^{\top}\nabla \mathbf{\hat{f}}(\hat{X}^k))\|_F \nonumber \\ =& \|\hat{X}^{k}  - \bar{X}_W^{k}  \|_F \nonumber + \gamma \cdot\|\nabla \mathbf{\hat{f}}(\hat{X}^k) - \frac1W\bm{1}_W\bm{1}_W^{\top}\nabla \mathbf{\hat{f}}(\hat{X}^k))\|_F.
		\end{align}
		For the last term on the right-hand side of \eqref{proof:thm2-8}, we have
		\begin{align}\label{proof:thm2-8-2}
			&\|\nabla \mathbf{\hat{f}}(\hat{X}^k) - \frac1W\bm{1}_W\bm{1}_W^{\top}\nabla \mathbf{\hat{f}}(\hat{X}^k))\|_F^2 \\
			=& \sum_{w \in \mathcal{W}} \|\nabla \hat{f}_w(\hat{x}_w^k) - \frac{1}{W} \sum_{w \in \mathcal{W}} \nabla \hat{f}_w(\hat{x}_w^k)\|^2 \nonumber \\
			\leq & 3\sum_{w \in \mathcal{W}}  \|\nabla \hat{f}_w(\hat{x}_w^k) - \nabla \hat{f}_w(\bar{x}_W^k)\|^2  \nonumber \\&+ 3\sum_{w \in \mathcal{W}}  \|\nabla \hat{f}_w(\bar{x}_W^k) - \frac{1}{W} \sum_{w \in \mathcal{W}} \nabla \hat{f}_w(\bar{x}_W^k)\|^2 \nonumber \\& +3\sum_{w \in \mathcal{W}}  \|\frac{1}{W} \sum_{w \in \mathcal{W}} \nabla \hat{f}_w(\bar{x}_W^k) - \frac{1}{W} \sum_{w \in \mathcal{W}} \nabla \hat{f}_w(\hat{x}_w^k)\|^2 \nonumber \\
			\leq & 6L^2 \|\hat{X}^k - \bar{X}_W^k\|_F^2 \nonumber\\& + 3\sum_{w \in \mathcal{W}}  \|\nabla \hat{f}_w(\bar{x}_W^k) - \frac{1}{W} \sum_{w \in \mathcal{W}} \nabla \hat{f}_w(\bar{x}_W^k)\|^2 \nonumber \\
			\leq & 6L^2 \|\hat{X}^k - \bar{X}_W^k\|_F^2 \nonumber \\&+ 6\sum_{w \in \mathcal{W}}  \Bigg(\|\nabla \hat{f}_w(\bar{x}_W^k) - \frac{1}{R} \sum_{w \in \mathcal{R}} \nabla \hat{f}_w(\bar{x}_W^k)\|^2  \nonumber \\&+ \|\frac{1}{R} \sum_{w \in \mathcal{R}} \nabla \hat{f}_w(\bar{x}_W^k) - \frac{1}{W} \sum_{w \in \mathcal{W}} \nabla \hat{f}_w(\bar{x}_W^k)\|^2\Bigg) \nonumber \\
			\leq&6L^2 \|\hat{X}^k - \bar{X}_W^k\|_F^2  + 6W(\delta^2 A^2 + \max\{\xi^2, A^2\}), \nonumber
		\end{align}
		where the second inequality is due to Assumption \ref{assump: L-smooth} and the last inequality is due to Assumptions \ref{assump: bounded heterogeneity} and \ref{assump: bounded disturbances}.
		
		Therefore, we have
		\begin{align}\label{proof:thm2-8-3}
			&\|\nabla \mathbf{\hat{f}}(\hat{X}^k) - \frac1W\bm{1}_W\bm{1}_W^{\top}\nabla \mathbf{\hat{f}}(\hat{X}^k))\|_F  \\ \leq&\sqrt{6}L \|\hat{X}^k - \bar{X}_W^k\|_F  + \sqrt{6W(\delta^2 A^2 + \max\{\xi^2, A^2\})}. \nonumber
		\end{align}
		Substituting \eqref{proof:thm2-8} and \eqref{proof:thm2-8-3} into \eqref{proof:thm2-7-2} yields
		\begin{align}\label{proof:thm2-9}
			&\|\hat{X}^{k+1} - \bar{X}_W^{k+1}\|_F \\\leq& (1 - \omega) (1 + \sqrt{6}L\gamma) \|\hat{X}^{k} - \bar{X}^{k}\|_F  \nonumber \\ &+ (1-\omega)\gamma\sqrt{6W(\delta^2 A^2 + \max\{\xi^2, A^2\})}\nonumber
		\end{align}
		Since $\gamma \leq \frac{\omega}{2\sqrt{6}L(1-\omega)}$, we obtain 1 + $\sqrt{6}L\gamma \leq  \frac{2 - \omega}{2(1-\omega)}$ and
		\begin{align}\label{proof:thm2-10}
			&\|\hat{X}^{k+1} - \bar{X}_W^{k+1}\|_F \\\leq& (1 - \frac{\omega}{2}) \|\hat{X}^{k} - \bar{X}_W^{k}\|_F \nonumber \\&+ (1-\omega)\gamma\sqrt{6W(\delta^2 A^2 + \max\{\xi^2, A^2\})}\nonumber\\\leq& (1 - \frac{\omega}{2})^{k+1} \|\hat{X}^{0} - \bar{X}_W^{0}\|_F \nonumber \\& + \frac{2(1-\omega)\gamma\sqrt{6W(\delta^2 A^2 + \max\{\xi^2, A^2\})}}{\omega}. \nonumber
		\end{align}
		Since $\hat{x}_w^0 = x^0$ for all $w \in \mathcal{W}$, we have $\|\hat{X}^0 - \bar{X}_W^0\|_F=0$ and thus
		\begin{align}\label{proof:thm2-11}
			&\|\hat{X}^{k+1} - \bar{X}_W^{k+1}\|_F^2 \\ \leq &\frac{24W(1-\omega)^2\gamma^2(\delta^2 A^2 + \max\{\xi^2, A^2\})}{\omega^2}. \nonumber
		\end{align}
		Substituting \eqref{proof:thm2-11} into \eqref{proof:thm2-6-2}, we have
		\begin{align} \label{proof:thm2-12}
			&\|\nabla f(\bar{x}^k)\|^2  \\ \le& \frac{4(f(\bar{x}_W^k) - f(\bar{x}_W^{k+1}))}{\gamma} \nonumber\\&+ \frac{144  L^2 W(1-\omega)^2\gamma^2(\delta^2 A^2 + \max\{\xi^2, A^2\})}{R\omega^2} + 4\delta^2A^2. \nonumber
		\end{align}
		Thus, with Assumption \ref{assump: lower boundedness}, we have
		\begin{align}\label{proof:thm2-13}
			&\frac{1}{K} \sum_{k=0}^{K-1}\|\nabla f(\bar{x}^k)\|^2 \\\leq&\frac{4(f(\bar{x}_W^0) - f^*)}{\gamma K} \nonumber \\&+ \frac{144L^2W(1-\omega)^2\gamma^2(\delta^2 A^2 + \max\{\xi^2, A^2\})}{R\omega^2} + 4\delta^2A^2 \nonumber \\
			=&\frac{4(f(x^0) - f^*)}{\sqrt{K}} \nonumber\\&+ \frac{144L^2W(1-\omega)^2(\delta^2 A^2 + \max\{\xi^2, A^2\})}{R\omega^2K} + 4\delta^2A^2 \nonumber \\
			=& O(\frac{f(x^0) - f^*}{\sqrt{K}}) + O(\frac{L^2(\lambda')^2(\delta^2 A^2 + \max\{\xi^2, A^2\})}{K(1 - \lambda')^2}) \nonumber\\& + O(\delta^2A^2), \nonumber
		\end{align}
		where the first equality is due to $\gamma = \frac{1}{\sqrt{K}}$.
		
		For the consensus error, we have
		\begin{align}
			&\max_{w \in \mathcal{R}}\|x_w^{k} - \bar{x}^{k}\|^2 \\
			\leq&\sum_{w \in \mathcal{R}}\|x_w^{k} - \bar{x}^{k}\|^2 \nonumber \\\leq & \sum_{w \in \mathcal{W}}\|\hat{x}_w^{k} - \bar{x}_{W}^{k} + \bar{x}_{W}^{k} -   \bar{x}^{k}\|^2 \nonumber\\ \leq
		 & \sum_{w \in \mathcal{W}}(2\|\hat{x}_w^{k} - \bar{x}_{W}^{k}\|^2 +2 \| \bar{x}_{W}^{k} -   \bar{x}^{k}\|^2)\nonumber \\
			\leq& 2 \|\hat{X}^k - \bar{X}_{W}^k\|_F^2 + 2W (\frac{1}{R} \sum_{w \in \mathcal{R}} \|x_w^k - \bar{x}_W^k\|^2) \nonumber\\
			\leq&  2(1+\frac{W}{R}) \|\hat{X}^k - \bar{X}_{W}^k\|_F^2 \nonumber\\ \leq &\frac{48W(1+\frac{W}{R})(1-\omega)^2(\delta^2 A^2 + \max\{\xi^2, A^2\})}{K\omega^2} \nonumber\\ =&  \frac{48W(1+\frac{W}{R})(\lambda')^2(\delta^2 A^2 + \max\{\xi^2, A^2\})}{K(1 - \lambda')^2} \nonumber \\ =& O(\frac{(\lambda')^2 (\delta^2A^2 + \max \{\xi^2, A^2\})}{K(1 - \lambda')^2}), \nonumber			
		\end{align}
		which completes the proof.
	\end{proof}

	\section*{Proof of Theorem \ref{thm: lower bound}}

	\begin{proof}
		Theorem \ref{thm: lower bound} consists of two parts. For any majority-dominant $(\rho, M)$-robust aggregator, we will establish \eqref{eq: lower bound for RAgg} as follows. For the lower bound \eqref{eq: lower bound for WeiMean} of the weighted mean aggregator, the proof is the same as that of Theorem 9 in \cite{peng2025mean} via setting the network topology $\mathcal{G}$ as a fully connected graph, and is thus omitted.
		
		The main idea of showing the lower bound of the majority-dominant $(\rho, M)$-robust aggregator is to construct two different problem instances over the same network topology, such that Algorithm \ref{algorithm:1} cannot distinguish between these two instances, leading to the same output. The difference between the two global costs of these two instances results in the lower bound of the learning error for any majority-dominant $(\rho, M)$-robust aggregator. The failure of the distinguishability arises from the fact that the majority-dominant $(\rho, M)$-robust aggregator is defined over a \textit{set} of neighboring vectors, and thus the order does not matter.
		
			

		Without loss of generality, let $\mathcal{W} = \{1, \cdots, W\}$ be the set of agents, among which $\mathcal{R} = \{1, \cdots, R\}$ is the set of regular agents. The samples of all agents have two possible labels, 1 or 2, which correspond to two different functions $f(x;1)$ and $f(x;2)$ with
		\begin{align}
			f(x;t) = \frac{(1 - \delta_{\max})c}{\sqrt{2}} [x]_t + \frac{L}{2} \|x\|^2, \ t \in \{1, 2\},
		\end{align}
		Here, $c \triangleq \min \{\xi, A\}$.
		
		For each agent $w \in \mathcal{W}$, there is only one sample with label $\hat{b}^{(w)}$. The two different sets of local costs share the same form of
		\begin{align}
			\hat{f}_w(x) = f(x; \hat{b}^{(w)}), \ \forall w \in \{1, \cdots, W\},
		\end{align}
		or equivalently
		\begin{align}
			f_w(x) &= f(x; b^{(w)}), \  \forall w \leq R, \\
			\tilde{f}_w(x) &= f(x; \tilde{b}^{(w)}), \ \forall w > R.
		\end{align}
		Now we construct two instances with different sets of labels.
		
		We first consider a typical case with $R=W-R=4$. The constructed network topology is shown in Figure \ref{fig:graph_lowerbound}. Note that all regular agents are fully connected and each has two poisoned neighbors. Therefore, the local contamination rate is $\delta_{\max}  = \max_{w \in \mathcal{R}} (1 - \frac{ {\bar{R}_w}}{ {\bar{N}_w}}) =  \frac{2}{R+2}$.
		
		\begin{figure}[ht!]
			\centering
			\includegraphics[scale=0.4]{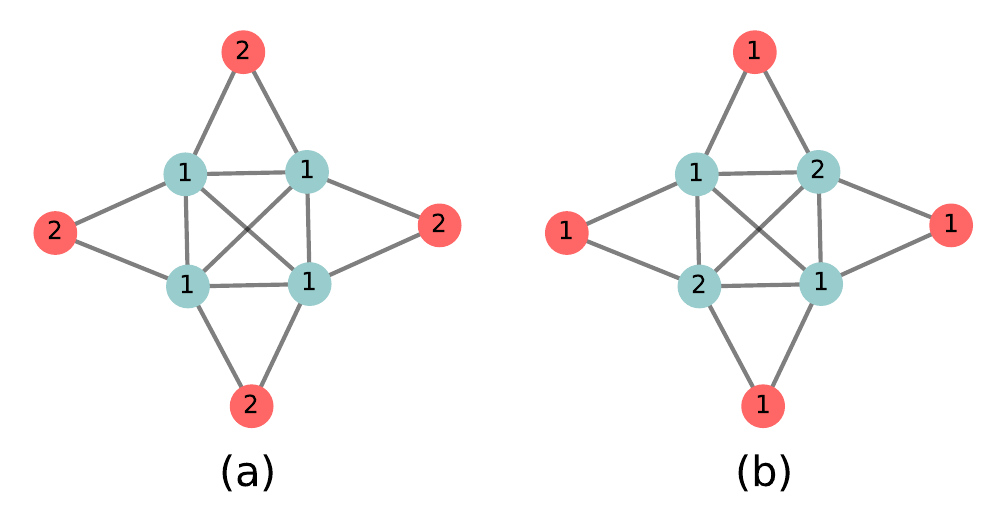}
			\caption{An 8-agent example of the constructed network topology where the regular agents are fully connected and each regular agent has two poisoned neighbors. The blues and reds represent the regular agents and the poisoned agents, respectively. (a) Instance 1: the local data of regular agents are with label 1 and the local data of poisoned agents are  with label 2; (b) Instance 2: the local data of regular agents $1$ and $3$ are with label 2 and the local data of other agents are with label 1.}
			\label{fig:graph_lowerbound}
		\end{figure}
		
		The first set of labels, denoted as $\{\hat{b}^{(w, 1)}, w \in \{1, \cdots, W\}\}$ is given by
		\begin{align}\label{proof:thm4-instance1}
			\hat{b}^{(w, 1)} = \left\{
			\begin{aligned}
				&1, &&w \leq R, \\
				&2, && w > R.
			\end{aligned}
			\right.
		\end{align}
		The second set of labels, denoted as $\{\hat{b}^{(w, 2)}, w \in \{1, \cdots, W\}\}$, is given by
		
		\begin{align}\label{proof:thm4-instance2}
			\hat{b}^{(w, 2)} = \left\{
			\begin{aligned}
				&1, &&w \neq 1 \text{ or }  3, \\
				&2, && w = 1 \text{ or } 3.
			\end{aligned}
			\right.
		\end{align}
		
		Denote $f^{(1)}(x) = \frac{1}{R} \sum_{w = 1}^R f(x; \hat{b}^{(w, 1)})$ and $f^{(2)}(x) = \frac{1}{R} \sum_{w = 1}^R f(x; \hat{b}^{(w, 2)})$ as the two global costs. We can check that all the assumptions are satisfied in these two instances. Since
		\begin{align}\label{proof:thm4-grad}
			\nabla f(x; t) = \frac{(1 - \delta_{\max})c}{\sqrt{2}} e_t + Lx,
		\end{align}
		where $e_t$ is the unit vector with the $t$-th element being 1. The gradients of $f^{(1)}$ and $f^{(2)}$ are respectively
		\begin{align}\label{proof:thm4-exp1-gradient}
			\nabla f^{(1)}(x) = \frac{(1 - \delta_{\max})c}{\sqrt{2}} e_1 + Lx,
		\end{align}
		\begin{align}\label{proof:thm4-exp2-gradient}
			\nabla f^{(2)} (x) = \frac{(1-2\delta_{\max})c}{\sqrt{2}} e_1 + \frac{\delta_{\max}c}{\sqrt{2}} e_2 + Lx.
		\end{align}
		As a result, we know their minima are achieved at $x^{*, (1)} = -\frac{(1 - \delta_{\max})c}{\sqrt{2}L} e_1$ and $x^{*, (2)} = - \frac{(1 - 2\delta_{\max})c}{\sqrt{2}L} e_1 - \frac{\delta_{\max}c}{\sqrt{2}L}e_2$, respectively, and there exists a uniform lower bound
		\begin{align}
			f^{(t)}(x) \geq f^* \triangleq - \frac{c^2}{2L},
		\end{align}
		satisfying Assumption \ref{assump: lower boundedness} for $t \in \{1, 2\}$.
		From \eqref{proof:thm4-exp1-gradient} and \eqref{proof:thm4-exp2-gradient}, the gradients are linear, Assumption \ref{assump: L-smooth} is satisfied and the constant is exactly $L$.
		Further, Assumption \ref{assump: bounded heterogeneity} is satisfied with constant $\xi$, as
		\begin{align}
			&\max_{w \leq R} \|\nabla f(x; \hat{b}^{(w, 1)}) - \nabla f^{(1)}(x)\| = 0, \\
			&\max_{w \leq R} \|\nabla f(x; \hat{b}^{(w, 2)}) - \nabla f^{(2)}(x)\| \\&= \max \{|1 - 2\delta_{\max}|, \delta_{\max}\} c \leq c \leq \xi, \nonumber
		\end{align}
		where the first inequality is due to $\delta_{\max} \in [0, 1]$ and the last inequality comes from $c \triangleq \min\{A, \xi\} \leq \xi$.
		Assumption \ref{assump: bounded disturbances} is satisfied with constant $A$, since
		\begin{align}
			\max_{w > R} \| \nabla f(x; \hat{b}^{(w, 1)}) - \nabla f^{(1)}(x)\| &= (1 - \delta_{\max})c \leq c\leq A, \\
			\max_{w > R} \| \nabla f(x; \hat{b}^{(w, 1)}) - \nabla f^{(1)}(x)\| &= \delta_{\max}c \leq c\leq A,
		\end{align}
		where the last inequalities comes from $c \triangleq \min\{A, \xi\} \leq A$.

		Next, we show the following claim: with the identical initialization $x^0$ over all agents, both instance will have the same local model after one iteration for both the regular and poisoned agents. For the first instance, with \eqref{proof:thm4-instance1} we have
		\begin{align}
			\hat{x}_w^{\frac{1}{2}, (1)} = \left\{
			\begin{aligned}
				&x^0 - \gamma \cdot \nabla f(x^0; 1), && w \leq R,\\
				&x^0 - \gamma \cdot \nabla f(x^0; 2), && w > R.
			\end{aligned}
			\right.
		\end{align}
		For notational convenience, denote $y_1 \triangleq x^0  - \gamma \cdot \nabla f(x^0; 1)$ and $y_2 \triangleq x^0 - \gamma \cdot \nabla f(x^0; 2)$. Thus, we have
		\begin{align}
			\hat{x}_w^{\frac{1}{2}, (1)} = \left\{
			\begin{aligned}
				&y_1, && w \leq R,\\
				&y_2, && w > R.
			\end{aligned}
			\right.
		\end{align}
		Further, since $\text{RAgg}(\cdot)$ is a $(\rho, M)$-robust aggregator, according to \eqref{eq: rho-robust aggregartor}, for regular agent $w \in \{1, \cdots, R\}$, we have
		\begin{align}\label{proof:thm4-ins1-regular-1}
			&\|\text{RAgg}(\{\hat{x}_v^{\frac{1}{2}, (1)}: v \in {\bar{\mathcal{N}}_w}\}) - \bar{x}_w^{\frac{1}{2}, (1)}  \| \\ \leq &
			\|\text{RAgg}(\{\underbrace{y_1, \cdots, y_1}_{R \text{ times}}, y_2, y_2\}) - \bar{x}_w^{\frac{1}{2}, (1)}  \| \nonumber \\ \leq
			&\rho \cdot \left[\max_{v \in {\bar{\mathcal{R}}_w}} \|\hat{x}_v^{\frac{1}{2}, (1)} -\bar{x}_w^{\frac{1}{2}, (1)} \|\right] \nonumber \\ = & 0, \nonumber
		\end{align}
		where the equality is due to $\bar{x}_w^{\frac{1}{2}, (1)} = \sum_{v^\prime \in {\bar{\mathcal{R}}_w}} M_{wv^\prime} \hat{x}_{v^\prime}^{\frac{1}{2}, (1)} = y_1 = \hat{x}_{v}^{\frac{1}{2}, (1)}, \forall w \in \{1, \cdots, R\},v \in {\bar{\mathcal{R}}_w}$. Thus, we have
		\begin{align}\label{proof:thm4-ins1-regular}
			\hat{x}_w^{1, (1)} = \text{RAgg}(\{\hat{x}_v^{\frac{1}{2}, (1)}: v \in {\bar{\mathcal{N}}_w}\}) =  \bar{x}_w^{\frac{1}{2}, (1)} = y_1.
		\end{align}
		
		
		Since $\text{RAgg}(\cdot)$ is also a majority-dominant aggregator, for poisoned agent $w \in \{R+1, \ldots, W\}$, we have
		\begin{align}\label{proof:thm4-ins1-poison}
			&\hat{x}_w^{1, (1)} = \text{RAgg}(\{\hat{x}_v^{\frac{1}{2}, (1)}: v \in {\bar{\mathcal{N}}_w}\})  \\&= \text{RAgg}(\{\{y_1: v \in \mathcal{R}_w \}\cup\{y_2\}\})= y_1, \nonumber
		\end{align}
		where the last equality is because that the local models of regular neighbors are the same and form a majority.

		For the second instance, with the same initialization at $x^0$, we have $\hat{x}_w^{\frac{1}{2}, (2)} = y_{\hat{b}^{(w, 2)}}.$ Then the set
		\begin{align}
			\{\hat{x}_w^{\frac{1}{2}, (2)}:w\in\mathcal{R}\} = \{y_2,y_2,\underbrace{y_1, \cdots, y_1}_{(R-2) \text{ times}}\}.
		\end{align}
		%
		Further, since $(\rho, M)$-robust aggregator is defined over the set of inputs, we can change the order of the inputs. From \eqref{proof:thm4-ins1-regular},
		for regular agent $w \in \{1, \cdots, R\}$, we have
		\begin{align}\label{proof:thm4-ins2-regular}
			\hat{x}_w^{1, (2)}
			=&\text{RAgg}(\{\hat{x}_v^{\frac{1}{2}, (2)}: v \in {\bar{\mathcal{N}}_w}\})  \\ = &
			\text{RAgg}(\{y_2,y_2,\underbrace{y_1, \cdots, y_1}_{(R-2) \text{ times}}, y_1, y_1\})  \nonumber\\ = &
			\text{RAgg}(\{\underbrace{y_1, \cdots, y_1}_{R \text{ times}}, y_2, y_2\})  \nonumber \\ =& y_1, \nonumber
		\end{align}
		where the last equality is because the inputs of $(\rho, M)$-robust aggregator $\text{RAgg}(\cdot)$ in \eqref{proof:thm4-ins1-regular} and \eqref{proof:thm4-ins2-regular} are the same ($R$ of them are $y_1$ and 2 of them are $y_2$), which leads to the same output.
		%
		%
		%
		
		Similarly, since for any poisoned agent $w \in \{R+1, \cdots, W\}$ the two instances has the same set of input messages, namely,
		\begin{align}
			\hspace{-1em}\{\hat{x}_v^{\frac{1}{2}, (2)}: v \in {\bar{\mathcal{N}}_w}\} = \{y_1,y_1,y_2\} = \{\hat{x}_v^{\frac{1}{2}, (1)}: v \in {\bar{\mathcal{N}}_w}\},
		\end{align}
		we have that
		\begin{align}\label{proof:thm4-ins2-poison}
			\hat{x}_w^{1, (2)} = \hat{x}_w^{1, (1)} = y_1.
		\end{align}
		
		
		Therefore, in the two constructed instances, the regular and poisoned agents will have the same local model at iteration $1$, given by
		\begin{align}
			\hat{x}_w^{1, (t)} = y_1, \quad  \forall w \in \mathcal{W}, \  t \in \{1, 2\}.
		\end{align}
		By induction, there must exist a sequence of vector $\{z_k:k\in\{1,\ldots,K\}\}$ such that
		\begin{align}
			\hat{x}_w^{k, (1)} = \hat{x}_w^{k, (2)} = z_k, \quad \forall w\in \mathcal{W},\ k \in \{1, \cdots, K\},
		\end{align}
		which means the Algorithm \ref{algorithm:1} with a $(\rho, M)$-robust aggregator $\text{RAgg}(\cdot)$ outputs the same average model of regular agents in these two constructed instances in any iteration, written as
		\begin{align}
			\bar{x}^{k, (1)} = \bar{x}^{k, (2)} = z_k, \quad \forall k \in \{1, \cdots, K\}.
		\end{align}
		
		
		However, according to \eqref{proof:thm4-exp1-gradient} and \eqref{proof:thm4-exp2-gradient}, the gradients of the two global costs in these two instances are different, which implies at least one of these two instances has high learning error with the same output. Below we investigate the lower bound of the learning error based on this observation.
		
		Specifically, we have
		\begin{align}
			&\max_{t \in \{1, 2\}} \|\nabla f^{(t)}(\bar{x}^k)\|\\
			=&\max_{t \in \{1, 2\}} \|\nabla f^{(t)}(z_k)\| \nonumber\\
			= & \max\{\|\nabla f^{(1)}(z_k)\|, \|\nabla f^{(2)}(z_k)\|\} \nonumber\\
			\geq & \frac{1}{2} \left(\|\nabla f^{(1)}(y_k)\| + \|\nabla f^{(2)}(z_k)\|\right) \nonumber \\
			\geq & \frac{1}{2} \|\nabla f^{(1)}(z_k) - \nabla f^{(2)}(z_k)\| \nonumber\\
			= & \frac{\delta_{\max}c}{2}. \nonumber
		\end{align}
		Further, since
		\begin{align}
			&2 \max_{t \in \{1, 2\}} \frac{1}{K} \sum_{k=1}^{K} \|\nabla f^{(t)} (\bar{x}^k)\|^2 \\
			\geq &\frac{1}{K} \sum_{k=1}^{K} \|\nabla f^{(1)} (\bar{x}^k)\|^2 + \frac{1}{K} \sum_{k=1}^{K} \|\nabla f^{(2)} (\bar{x}^k)\|^2 \nonumber \\
			\geq& \frac{1}{K} \sum_{k=1}^{K} (\max_{t \in \{1, 2\}} \|\nabla f^{(t)}(\bar{x}^k)\|^2) \nonumber \\
			\geq& \frac{\delta_{\max}^2 c^2}{4}, \nonumber
		\end{align}
		we have
		\begin{align}\label{proof:thm4-lower-bound}
			\max_{t \in \{1, 2\}} \frac{1}{K} \sum_{k=1}^K \|\nabla f^{(t)}(\bar{x}^k)\|^2  \geq
			\frac{\delta_{\max}^2c^2}{8}.
		\end{align}
		Choosing $R$ regular local costs $\{f_w(x) = f(x; \hat{b}^{(w, t)}): w \leq R\}$ and $W - R$ poisoned local costs $\{\tilde{f}_w(x) = f(x; \hat{b}^{(w, t)}): w > R\}$ where $t = \arg\max_{t \in \{1, 2\}}\frac{1}{K} \sum_{k=1}^K \|\nabla f^{(t)}(\bar{x}^k)\|^2$, \eqref{proof:thm4-lower-bound} demonstrates that
		\begin{align}
			\frac{1}{K} \sum_{k=1}^{K} \|\nabla f(\bar{x}^k)\| \geq \frac{\delta_{\max}^2 c^2}{8} = \Omega(\delta_{\max}^2 \min\{A^2, \xi^2\}),
		\end{align}
		which proves \eqref{eq: lower bound for RAgg}.
		
		In general, given any $R$ and $W$, the above arguments hold true, if:
		\begin{itemize}
			\item[(a)] All regular agents are fully connected, and each has two poisoned neighbors.
			\item[(b)] For the first instance, the local data of regular agents are with label 1 and the local data of poisoned agents are with label 2; For the second instance, the local data of two regular agents are with label 2 and the local data of other agents are with label 1. The two regular agents with label 2 do not share a common poisoned neighbor.
			\item[(c)] Similar to Lemma 5 in \cite{peng2025mean}, a $(\rho, M)$-robust aggregator can exist only if the fraction of poisoned inputs is less than $\frac{1}{2}$. Thus each poisoned agent has at least two regular neighbors.
			
		\end{itemize}
		Let $\{1,\ldots, R\}$ be the set of regular agents, where agents $\{1,R\}$ change their labels to 2 in the second instance.
		Suppose an integer feasible solution $\{p,q,n_i\in\mathbb{N}:i\in\{1,2,\ldots,p+q\}\}$ satisifes the following constraints:
		\begin{align*}
			\begin{cases}
				n_1+\ldots+n_p =n_{p+1}+\ldots+n_{p+q} =  R, \ p+q = W-R,\\
				p,q\ge 2,\quad n_i\ge 2,\ \forall i.
			\end{cases}
			\	\end{align*}
		Then the network topology is constructed as follows:
		\begin{itemize}
			\item All regular agents are fully connected.
			\item Poisoned agents does not connect with each other.
			\item Let $S_i= \sum_{j< i}n_j$ be the index pointers. Then the $i$-th poisoned agent is connected with regular agents $\{w: (S_i \mod R)+1\le w\le (S_{i+1}-1 \mod R)+1\}$.
		\end{itemize}
		It can be easily checked that all three conditions in the network topology are satisfied, and a feasible solution exists if $R \ge 4$. Further, if $R$ is even, there exists a feasible solution if and only if $R+4 \leq W \leq 2R$; if R is odd, the sufficient and necessary condition changes to $R+4 \leq W \leq 2R-1$.
		
		Otherwise, if $R \ge 4$ is an even number, and $W > 2R$, we first construct the network with $R$ regular agents and $R$ poisoned agents, and then add an additional $W - 2R$ ``dummy'' poisoned agents that have label 1 in both instances and whose neighbors are all regular agents; if $R \ge 4$ is an odd number, and $W > 2R - 1$, we first construct the network with $R$ regular agents and $R - 1$ poisoned agents, and then add an additional $W - 2R + 1$ ``dummy'' poisoned agents that have label 1 in both instances and whose neighbors are all regular agents. Then the proof of \eqref{eq: lower bound for RAgg} is still valid.

		For special cases excluded in the above construction, if $R\ge 4$ and $W\le 7$, one can show that there is no network satisfying (a)-(c) by pigeon-hole principle, but we can slightly change the construction: the regular agents are fully connected, and each poisoned agent is connected with all the regular agents (namely, $\mathcal{R}_w=\mathcal{R},\ \forall w\not\in\mathcal{R}$). For the first instance, all agents are with label 1 except for a poisoned agent; On the other hand, for the second instance, the only agent that has label 2 is a regular agent. Then slightly changing the analysis gives a different constant instead of 8 in \eqref{proof:thm4-lower-bound}, demonstrating the same order of the lower bound.
		The cases of $R=3$ and $R=2$ share the same construction as above.
		%
		%
		%
		If $R=1$, there is no poisoned neighbor due to the existence of $(\rho, M)$-robust aggregator
		mentioned in (c). Therefore, we have fully extended the result of $R=W-R=4$ to any feasible value of $R$ and $W.$
	\end{proof}

	\bibliographystyle{IEEEtran}
	\bibliography{refs}
	
	%
	%
	%
	%
	%
	%
	%
	%
	%
	
	\newpage
	\twocolumn[
	\begin{@twocolumnfalse}
		\section*{\centering{Supplementary Material for \\ \emph{Topology-Independent Robustness of  the Weighted Mean under Label Poisoning Attacks \\in Heterogeneous Decentralized Learning \\[15pt]}}}
	\end{@twocolumnfalse}
	]
	\setcounter{section}{0}
	\pagenumbering{arabic}

	\section*{Proof of Theorem \ref{thm: RAgg}}
	
	\begin{proof}
		Letting $\tilde{x}^k \triangleq \sum_{w \in \mathcal{R}} p_w x_w^k$, according to Assumption \ref{assump: L-smooth}, we have
		\begin{align}\label{proof:thm1-1}
			 & f(\tilde{x}^{k+1}) \\
        \leq & f(\tilde{x}^k) + \langle \nabla f (\tilde{x}^k), \tilde{x}^{k+1} - \tilde{x}^k\rangle + \frac{L}{2} \|\tilde{x}^{k+1} - \tilde{x}^k\|^2. \notag
		\end{align}
		For the second term on the right-hand side of \eqref{proof:thm1-1}, it holds that
		\begin{align} \label{proof:thm1-2}
			&\langle \nabla f(\tilde{x}^k), \tilde{x}^{k+1} - \tilde{x}^k \rangle \\ =& \gamma \langle \nabla f(\tilde{x}^k), \frac{1}{\gamma}(\tilde{x}^{k+1} - \tilde{x}^k) \rangle \nonumber \\ =& \frac{\gamma}{2} \|\nabla f(\tilde{x}^k) + \frac{1}{\gamma}(\tilde{x}^{k+1} - \tilde{x}^k)\|^2 - \frac{\gamma}{2} \|\nabla f(\tilde{x}^k)\|^2 \nonumber \\ &- \frac{\gamma}{2} \|\frac{1}{\gamma}(\tilde{x}^{k+1} - \tilde{x}^k)\|^2. \nonumber
		\end{align}
		With $\gamma \leq \frac{1}{L}$, combining \eqref{proof:thm1-1} and \eqref{proof:thm1-2} yields
		\begin{align} \label{proof:thm1-3}
			&f(\tilde{x}^{k+1}) \\ \leq & \ f(\tilde{x}^k) + \frac{\gamma}{2} \|\nabla f(\tilde{x}^k) + \frac{1}{\gamma}(\tilde{x}^{k+1} - \tilde{x}^k)\|^2 - \frac{\gamma}{2} \|\nabla f(\tilde{x}^k)\|^2 \nonumber\\ & + (\frac{L\gamma^2}{2} - \frac{\gamma}{2}) \|\frac{1}{\gamma}(\tilde{x}^{k+1} - \tilde{x}^k)\|^2 \nonumber \\\leq &\ f(\tilde{x}^k) + \frac{\gamma}{2} \|\nabla f(\tilde{x}^k) + \frac{1}{\gamma}(\tilde{x}^{k+1} - \tilde{x}^k)\|^2 - \frac{\gamma}{2} \|\nabla f(\tilde{x}^k)\|^2. \nonumber
		\end{align}
		
		Now we handle the second term on the right-hand side of \eqref{proof:thm1-3}. Note that
		\begin{align}\label{proof:thm1-4}
			&\|\nabla f(\tilde{x}^k) + \frac{1}{\gamma} (\tilde{x}^{k+1} - \tilde{x}^k)\| \\ \leq &
			\|\nabla f(\tilde{x}^k) - \sum_{w \in \mathcal{R}} p_w \nabla f_w(\tilde{x}^k)\| \nonumber\\ & + \|\sum_{w \in \mathcal{R}} p_w (\nabla f_w(\tilde{x}^k) - \nabla f_w(x_w^k))\| \nonumber  \\ &+ \|\sum_{w \in \mathcal{R}} p_w(\nabla f_w(x_w^k) + \frac{1}{\gamma} (x_w^{k+1} - x_w^k))\|. \nonumber
		\end{align}
		For the first term on the right-hand side of \eqref{proof:thm1-4}, with Assumption \ref{assump: bounded heterogeneity}, we have
		\begin{align}\label{proof:thm1-5}
			&\|\nabla f(\tilde{x}^k) - \sum_{w \in \mathcal{R}} p_w \nabla f_w(\tilde{x}^k)\| \\=& \|\sum_{w \in \mathcal{R}} (\frac{1}{R} - p_w) \nabla f_w(\tilde{x}^k)\|\nonumber \\=& \|\sum_{w \in \mathcal{R}} (\frac{1}{R} - p_w) (\nabla f_w(\tilde{x}^k) - \nabla f(\tilde{x}^k))\| \nonumber\\ \leq & \sum_{w \in \mathcal{R}} |\frac{1}{R} - p_w| \|\nabla f_w(\tilde{x}^k) - \nabla f(\tilde{x}^k)\| \leq  \beta \xi. \nonumber
		\end{align}
		For the second term on the right-hand side of \eqref{proof:thm1-4}, we have
		\begin{align}\label{proof:thm1-6}
			&\|\sum_{w \in \mathcal{R}} p_w (\nabla f_w(\tilde{x}^k) - \nabla f_w(x_w^k))\| \\ \leq& \sum_{w \in \mathcal{R}} p_w \|\nabla f_w(\tilde{x}^k) - \nabla f_w(x_w^k)\| \nonumber \\ \leq &
			\sum_{w \in \mathcal{R}} p_w L \|\tilde{x}^k - x_w^k\| \nonumber \\ \leq &
			L \max_{w \in \mathcal{R}}\|\tilde{x}^k - x_w^k\|. \nonumber
		\end{align}
        For the third term on the right-hand side of \eqref{proof:thm1-4}, we have
		\begin{align}\label{proof:thm1-7}
			&\|\sum_{w \in \mathcal{R}} p_w(\nabla f_w(x_w^k) + \frac{1}{\gamma} (x_w^{k+1} - x_w^k))\| \\
			\overset{\eqref{eq: local update}}{=} & \|\sum_{w \in \mathcal{R}} p_w (\frac{1}{\gamma} (x_w^{k+1} - x_w^{k+\frac{1}{2}}))\| \nonumber\\
			=  &\frac{1}{\gamma} \|\sum_{w \in \mathcal{R}} p_w  (x_w^{k+1} - \sum_{v \in {\bar{\mathcal{R}}_w}} M_{wv}x_{v}^{k+\frac{1}{2}} \nonumber \\&+ \sum_{v \in {\bar{\mathcal{R}}_w}} M_{wv}x_{v}^{k+\frac{1}{2}}- x_w^{k+\frac{1}{2}})\| \nonumber \\
			\leq   &\frac{1}{\gamma} \sum_{w \in \mathcal{R}} p_w\|  x_w^{k+1} - \sum_{v \in  {\bar{\mathcal{R}}_w}} M_{wv}x_{v}^{k+\frac{1}{2}}\| \nonumber \\& + \frac{1}{\gamma}\|\sum_{w \in \mathcal{R}} p_w\sum_{v \in {\bar{\mathcal{R}}_w}} M_{wv}x_{v}^{k+\frac{1}{2}}- \sum_{w \in \mathcal{R}} p_w x_w^{k+\frac{1}{2}})\| \nonumber
			\\ \overset{(a)}{=}& \frac{1}{\gamma} \sum_{w \in \mathcal{R}} p_w\|  x_w^{k+1} - \sum_{v \in {\bar{\mathcal{R}}_w}} M_{wv}x_{v}^{k+\frac{1}{2}}\| \nonumber
			\\ \overset{\eqref{eq: rho-robust aggregartor}}{\leq}& \frac{\rho}{\gamma} \max_{w \in \mathcal{R}} \max_{v \in {\bar{\mathcal{R}}_w}} \|x_v^{k+\frac{1}{2}} - \sum_{u \in {\bar{\mathcal{R}}_w}} M_{wu}x_{u}^{k+\frac{1}{2}}\| \nonumber \\
			\overset{(b)}{\leq}& \frac{2\rho}{\gamma} \max_{w \in \mathcal{R}}  \|x_w^{k+\frac{1}{2}} - \tilde{x}^{k+\frac{1}{2}}\| \nonumber,
		\end{align}
		within which (a) is due to $\sum_{w \in \mathcal{R}} p_w\sum_{v \in {\bar{\mathcal{R}}_w}} M_{wv}x_{v}^{k+\frac{1}{2}} =  \sum_{w \in \mathcal{R}} p_w x_w^{k+\frac{1}{2}}$ as $p^{\top} M = p^{\top}$, and (b) comes from
		\begin{align}\label{proof:thm1-8}
			&\|x_v^{k+\frac{1}{2}} - \sum_{u \in {\bar{\mathcal{R}}_w}} M_{wu}x_{u}^{k+\frac{1}{2}} \| \\\leq & \|x_v^{k+\frac{1}{2}} - \tilde{x}^{k+\frac{1}{2}} \| + \| \sum_{u \in {\bar{\mathcal{R}}_w}} M_{wu}( \tilde{x}^{k+\frac{1}{2}} - x_{u}^{k+\frac{1}{2}})\| \nonumber \\  \leq & 2\max_{w \in \mathcal{R}} \|x_w^{k+\frac{1}{2}} - \tilde{x}^{k+\frac{1}{2}}\|.\nonumber
		\end{align}
		Combining \eqref{proof:thm1-5}, \eqref{proof:thm1-6} and \eqref{proof:thm1-7}, we have
		\begin{align}\label{proof:thm1-9}
			&\|\nabla f(\tilde{x}^k) + \frac{1}{\gamma} (\tilde{x}^{k+1} - \tilde{x}^k)\| \\ \leq& \beta\xi + L \max_{w \in \mathcal{R}} \|x_w^k - \tilde{x}^k \| + \frac{2\rho}{\gamma} \max_{w \in \mathcal{R}} \|x_w^{k+\frac{1}{2}} - \tilde{x}^{k+ \frac{1}{2}}\|. \nonumber
		\end{align}
		
		With Assumption \ref{assump: bounded heterogeneity}, it holds that
		\begin{align}\label{proof:thm1-10}
			&\|x_w^{k+\frac{1}{2}} - \tilde{x}^{k+\frac{1}{2}}\| \\
			=& \|(x_w^k - \tilde{x}^k) - \gamma\cdot(\nabla f_w(x_w^k) -  \sum_{v \in {\mathcal{R}}} p_v \nabla f_v(x_v^k))\| \nonumber \\
			\leq & \|x_w^k - \tilde{x}^k\| + \gamma \|\nabla f_w(x_w^k) -  \sum_{v \in {\mathcal{R}}} p_v \nabla f_v(x_v^k)\| \nonumber \\
			=& \|x_w^k - \tilde{x}^k\| + \gamma \|(1 - p_w) \nabla f_w(x_w^k) - \sum_{v \neq w, v \in {\mathcal{R}}} p_v \nabla f_v(x_v^k)\|\nonumber \\
			= & \|x_w^k - \tilde{x}^k\| + \gamma \|(1 - p_w) \nabla f_w(x_w^k) - (1 - p_w) \nabla f(x_w^k) \nonumber\\
			& - \sum_{v \neq w, v \in {\mathcal{R}}} p_v \nabla f_v(x_v^k) +  \sum_{v \neq w, v \in {\mathcal{R}}} p_v \nabla f(x_v^k) \nonumber \\
			& + (1 - p_w) \nabla f(x_w^k) - (1 - p_w) \nabla f(\tilde{x}^k) \nonumber\\
			& - \sum_{v \neq w, v \in {\mathcal{R}}} p_v \nabla f(x_v^k) +  \sum_{v \neq w, v \in {\mathcal{R}}} p_v \nabla f(\tilde{x}^k) \| \nonumber \\
			\leq& \| x_{w}^k - \tilde{x}^k\| + \gamma \cdot \| (1 - p_{w}) \left( \nabla f_{w}(x_{w}^k) - \nabla f(x_w^k) \right) \nonumber\\& - \sum_{v \ne w,  v\in \mathcal{R}} p_v \left( \nabla f_v(x_v^k) - \nabla f(x_v^k) \right) \| \nonumber \\
			& + \gamma \cdot \| (1 - p_{w}) \left( \nabla f(x_{w}^k) - \nabla f(\tilde{x}^k) \right) \nonumber \\&- \sum_{v \ne \omega, v \in \mathcal{R}} p_v \left( \nabla f(x_v^k) - \nabla f(\tilde{x}^k) \right)\| \nonumber \\
			\leq & \max_{w \in \mathcal{R}} \|x_w^k - \tilde{x}^k\| \nonumber \\ &+ \gamma (1 - p_w + \sum_{v\neq w, v\in \mathcal{R}} p_v) (\xi + L \max_{w \in \mathcal{R}} \|x_w^k - \tilde{x}^k\|) \nonumber\\
			\leq & (1 + 2\gamma L) \max_{w \in \mathcal{R}} \|x_w^k - \tilde{x}^k\| + 2\gamma \xi. \nonumber
		\end{align}
		Therefore, we obtain
		\begin{align}\label{proof:thm1-11}
			\max_{w \in \mathcal{R}} \|x_w^{k+\frac{1}{2}} - \tilde{x}^{k+\frac{1}{2}}\| \leq  (1 + 2\gamma L) \max_{w \in \mathcal{R}} \|x_w^k - \tilde{x}^k\| + 2\gamma \xi.
		\end{align}
		Substituting \eqref{proof:thm1-11} into \eqref{proof:thm1-9} results in
		\begin{align}\label{proof:thm1-12}
			&\|\nabla f(\tilde{x}^k) + \frac{1}{\gamma} (\tilde{x}^{k+1} - \tilde{x}^k)\| \\
			\leq& \beta\xi + L \max_{w \in \mathcal{R}} \|x_w^k - \tilde{x}^k \|\nonumber \\& + \frac{2\rho}{\gamma} \left[(1 + 2\gamma L) \max_{w \in \mathcal{R}} \|x_w^k - \tilde{x}^k\| + 2\gamma \xi\right] \nonumber \\
			=& ((1+4\rho)L + \frac{2\rho}{\gamma}) \max_{w \in \mathcal{R}} \|x_w^k - \tilde{x}^k\| + (\beta + 4\rho) \xi. \nonumber
		\end{align}
		Plugging \eqref{proof:thm1-12} back to \eqref{proof:thm1-3}, we have
		\begin{align}\label{proof:thm1-13}
			&\|\nabla f(\tilde{x}^k)\|^2 \\ \leq &\frac{2 (f(\tilde{x}^k) - f(\tilde{x}^{k+1}))}{\gamma} \nonumber \\&+ \left[((1+4\rho)L + \frac{2\rho}{\gamma}) \max_{w \in \mathcal{R}} \|x_w^k - \tilde{x}^k\| + (\beta + 4\rho) \xi\right]^2. \nonumber
		\end{align}

		
		Stepping from \eqref{proof:thm1-13}, since
		\begin{align}\label{proof:thm1-14}
			&\|\nabla f(\bar{x}^k)\|^2\\
			\leq & 2\|\nabla f(\bar{x}^k) - \nabla f(\tilde{x}^k)\|^2 + 2\|\nabla f(\tilde{x}^k)\|^2 \nonumber \\ \leq & 2 L^2 \|\bar{x}^k - \tilde{x}^k\|^2 + 2\|\nabla f(\tilde{x}^k)\|^2 \nonumber\\
			\leq & 2 L^2 \max_{w \in \mathcal{R}} \|x_w^k - \tilde{x}^k\|^2 + 2\|\nabla f(\tilde{x}^k)\|^2, \nonumber
		\end{align}
		we further have
		\begin{align}\label{proof:thm1-15}
			&\|\nabla f(\bar{x}^k)\|^2 \\ \leq &
			\frac{4 (f(\tilde{x}^k) - f(\tilde{x}^{k+1}))}{\gamma} + 2 L^2 \max_{w \in \mathcal{R}} \|x_w^k - \tilde{x}^k\|^2\nonumber \\& + 2\left[((1+4\rho)L + \frac{2\rho}{\gamma}) \max_{w \in \mathcal{R}} \|x_w^k - \tilde{x}^k\| + (\beta + 4\rho) \xi\right]^2 . \nonumber
		\end{align}
		
		Now we already know that $\|\nabla f(\bar{x}^k)\|$ is bounded by $\max_{w \in \mathcal{R}} \|x_w^k - \tilde{x}^k\|$, in addition to other terms. Next, we recursively bound this term. Note that
		\begin{align}\label{proof:thm1-16}
			&\|x_w^{k+1} - \tilde{x}^{k+1}\| \\\leq & \|x_w^{k+1} - \sum_{v \in {\bar{\mathcal{R}}_w}} M_{wv} x_v^{k+\frac{1}{2}}\| \nonumber\\& + \|\sum_{v \in {\bar{\mathcal{R}}_w}} M_{wv} x_v^{k+\frac{1}{2}} - \tilde{x}^{k+\frac{1}{2}}\| + \|\tilde{x}^{k+1} - \tilde{x}^{k+\frac{1}{2}}\|. \nonumber
		\end{align}
		Below, we handle the terms on the right-hand side of \eqref{proof:thm1-16} one by one.
		
		For the first term on the right-hand side of \eqref{proof:thm1-16}, since $x_w^{k+1} = \text{RAgg}( \{x_v^{k+\frac{1}{2}}:v \in {\bar{\mathcal{N}}_w}\})$, according to \eqref{eq: rho-robust aggregartor}, we know
		\begin{align}\label{proof:thm1-17}
			& \|x_w^{k+1} - \sum_{v \in {\bar{\mathcal{R}}_w}} M_{wv} x_v^{k+\frac{1}{2}}\| \\ \leq
			&\rho \cdot \max_{v \in {\bar{\mathcal{R}}_w}} \|x_v^{k+\frac{1}{2}} - \sum_{u \in {\bar{\mathcal{R}}_w}} M_{wu} x_u^{k+\frac{1}{2}}\| \nonumber \\
			\leq & 2\rho \cdot \max_{w \in \mathcal{R}} \|x_w^{k+\frac{1}{2}} - \tilde{x}^{k+\frac{1}{2}}\|, \nonumber
		\end{align}
		where the last inequality is due to \eqref{proof:thm1-9}. For the second term on the right-hand side of \eqref{proof:thm1-16}, we have
		\begin{align}\label{proof:thm1-18}
			&\|\sum_{v \in  {\bar{\mathcal{R}}_w}} M_{wv} x_v^{k+\frac{1}{2}} - \tilde{x}^{k+\frac{1}{2}}\|\\
			=& \|\sum_{v \in \mathcal{R}} (M_{wv} - p_v) x_v^{k+\frac{1}{2}}\| \nonumber \\
			=&\|\sum_{v \in \mathcal{R}} (M_{wv} - p_v) (x_v^{k+\frac{1}{2}} - \tilde{x}^{k+\frac{1}{2}})\| \nonumber\\
			\leq & \sum_{v \in \mathcal{R}} |M_{wv} - p_v|\| x_v^{k+\frac{1}{2}} - \tilde{x}^{k+\frac{1}{2}}\| \nonumber \\
			\leq & \max_{w \in \mathcal{R}}\sum_{v \in \mathcal{R}} |M_{wv} - p_v| \cdot \max_{v \in \mathcal{R}}\| x_v^{k+\frac{1}{2}} - \tilde{x}^{k+\frac{1}{2}}\| \nonumber \\
			= & \lambda \cdot \max_{v \in \mathcal{R}}\| x_v^{k+\frac{1}{2}} - \tilde{x}^{k+\frac{1}{2}}\|. \nonumber
		\end{align}
		For the third term on the right-hand side of \eqref{proof:thm1-16}, we have
		\begin{align}\label{proof:thm1-19}
			&\|\tilde{x}^{k+1} - \tilde{x}^{k+\frac{1}{2}}\| \\
			=& \|\sum_{w \in \mathcal{R}}p_w x_w^{k+1} - \sum_{w \in \mathcal{R}} p_w x_w^{k+\frac{1}{2}}\| \nonumber\\
						=& \|\sum_{w \in \mathcal{R}}p_w x_w^{k+1} - \sum_{w, v \in \mathcal{R}} p_w M_{wv} x_w^{k+\frac{1}{2}}\| \nonumber \\
									=& \|\sum_{w \in \mathcal{R}}p_w (x_w^{k+1} - \sum_{v \in \mathcal{R}} M_{wv} x_w^{k+\frac{1}{2}})\| \nonumber \\
												\leq& \sum_{w \in \mathcal{R}}p_w\| x_w^{k+1} - \sum_{v \in \mathcal{R}}  M_{wv} x_w^{k+\frac{1}{2}}\| \nonumber \\
			\leq &  2\rho \cdot \max_{w \in \mathcal{R}} \|x_w^{k+\frac{1}{2}} - \tilde{x}^{k+\frac{1}{2}}\|, \nonumber
		\end{align}
		where the last inequality is due to \eqref{proof:thm1-17}. Substituting \eqref{proof:thm1-17}, \eqref{proof:thm1-18} and \eqref{proof:thm1-19} into \eqref{proof:thm1-16}, we have
		\begin{align}\label{proof:thm1-20}
			\hspace{-1em}\max_{w \in \mathcal{R}}\|x_w^{k+1} - \tilde{x}^{k+1}\| \leq (4\rho + \lambda) \max_{w \in \mathcal{R}} \|x_w^{k+\frac{1}{2}} - \tilde{x}^{k+\frac{1}{2}}\|.
		\end{align}
		
		Let $\omega \triangleq 1 - (4\rho + \lambda)$. Since $\rho < \frac{1 - \lambda}{4}$ and $\rho, \lambda \geq 0$, we have $\omega \in (0,1]$ and
		\begin{align}\label{proof:thm1-21}
			&\max_{w \in \mathcal{R}}\|x_w^{k+1} - \tilde{x}^{k+1}\| \\ &\leq (1 - \omega) \max_{w \in \mathcal{R}} \|x_w^{k+\frac{1}{2}} - \tilde{x}^{k+\frac{1}{2}}\| \nonumber \\ &\leq (1 - \omega) (1 + 2\gamma L) \max_{w \in \mathcal{R}} \|x_w^k - \tilde{x}^k\| + 2(1 - \omega )\gamma \xi. \nonumber
		\end{align}
		Since $\gamma \leq \frac{\omega}{4(1 - \omega)L}$, we have $1 + 2\gamma L \leq \frac{2 - \omega}{2(1 - \omega)}$ and
		\begin{align}\label{proof:thm1-22}
			&\max_{w \in \mathcal{R}} \|x_w^{k+1} - \tilde{x}^{k+1}\| \\
			\leq& (1 - \frac{\omega}{2}) \max_{w \in \mathcal{R}} \|x_w^k - \tilde{x}^k\| + 2(1 - \omega)\gamma \xi \nonumber\\
			\leq& (1 - \frac{\omega}{2})^{k+1} \max_{w \in \mathcal{R}} \|x_w^0 - \tilde{x}^0\| + 2(1 - \omega) \gamma \xi \cdot \sum_{k'=0}^k(1 - \frac{\omega}{2})^{k'} \nonumber \\
			\leq &(1 - \frac{\omega}{2})^{k+1} \max_{w \in \mathcal{R}} \|x_w^0  - \tilde{x}^0\| \nonumber \\&+ 2(1 - \omega) \gamma \xi \cdot \frac{1 - (1 - \frac{\omega}{2})^{k+1}}{\frac{\omega}{2}} \nonumber \\
			\leq& (1 - \frac{\omega}{2})^{k+1} \max_{w \in \mathcal{R}} \|x_w^0  - \tilde{x}^0\| + \frac{4(1 - \omega) \gamma \xi }{\omega}. \nonumber
		\end{align}
		Since $x_w^0 = x_v^0$ for any $w , v \in \mathcal{R}$, we have $x_w^0 = \tilde{x}^0$ and thus
		\begin{align}\label{proof:thm1-23}
			\max_{w \in \mathcal{R}} \|x_w^{k} - \tilde{x}^k\| \leq \frac{4(1 - \omega)\gamma \xi}{\omega}.
		\end{align}
		
		Substituting \eqref{proof:thm1-23} into \eqref{proof:thm1-15}, we have
		\begin{align}\label{proof:thm1-24}
			&\|\nabla f(\bar{x}^k)\|^2 \\ \leq &
			\frac{4 (f(\tilde{x}^k) - f(\tilde{x}^{k+1}))}{\gamma} \nonumber \\& + 2\left[((1+4\rho)L + \frac{2\rho}{\gamma}) \frac{4(1 - \omega)\gamma \xi}{\omega}  + (\beta + 4\rho) \xi\right]^2 \nonumber \\& + 2 L^2 (\frac{4(1 - \omega)\gamma \xi}{\omega})^2 \nonumber \\
        \leq &
			\frac{4 (f(\tilde{x}^k) - f(\tilde{x}^{k+1}))}{\gamma} \nonumber \\&+ 2\left[ \frac{4(1+4\rho)L(1 - \omega)\gamma\xi}{\omega} + \frac{8(1 - \omega)\rho\xi}{\omega} + (\beta + 4\rho) \xi\right]^2 \nonumber \\& + \frac{32L^2(1 - \omega)^2\gamma^2 \xi^2}{\omega^2} \nonumber \\
            \leq &
			\frac{4 (f(\tilde{x}^k) - f(\tilde{x}^{k+1}))}{\gamma} +  \frac{64(1+4\rho)^2L^2(1 - \omega)^2\gamma^2\xi^2}{\omega^2} \nonumber \\& + \frac{256(1 - \omega)^2\rho^2\xi^2}{\omega^2} + 4\beta^2\xi^2 + 64\rho^2\xi^2 \nonumber\\& + \frac{32L^2(1 - \omega)^2\gamma^2 \xi^2}{\omega^2} \nonumber\\ \leq &
			\frac{4 (f(\tilde{x}^k) - f(\tilde{x}^{k+1}))}{\gamma} +  \frac{1024(1+\rho)^2L^2(1 - \omega)^2\gamma^2\xi^2}{\omega^2} \nonumber \\& + \frac{256\rho^2\xi^2}{\omega^2} + 4\beta^2\xi^2. \nonumber
		\end{align}
		Thus, with Assumption \ref{assump: lower boundedness}, we have
		\begin{align}
			&\frac{1}{K} \sum_{k=0}^{K-1} \|\nabla f(\bar{x}^k)\|^2 \\ \leq &\frac{2 (f(x^0) - f^*)}{\gamma K} +  \frac{1024(1+\rho)^2L^2(1 - \omega)^2\gamma^2\xi^2}{\omega^2} \nonumber \\& + \frac{256\rho^2\xi^2}{\omega^2} + 4\beta^2\xi^2 \nonumber \\
			= & \frac{2 (f(x^0) - f^*)}{\sqrt{K}} +  \frac{1024(1+\rho)^2L^2(1 - \omega)^2\xi^2}{K\omega^2} \nonumber \\& + \frac{256\rho^2\xi^2}{\omega^2} + 4\beta^2\xi^2 \nonumber \\
			=& O (\frac{f(x^0) - f^*}{\sqrt{K}}) + O(\frac{(1 +\rho)^2 L^2 (\lambda + \rho)^2 \xi^2}{K(1 - \lambda - 4\rho)^2}) \nonumber \\&+ O((\frac{\rho^2}{(1 - \lambda - 4\rho)^2} + \beta^2)\xi^2), \nonumber
		\end{align}
		where the first equality is due to $\gamma = \frac{1}{\sqrt{K}}$.
		
		For the consensus error, we have
		\begin{align}
			&\max_{w \in \mathcal{R}} \|x_w^k - \bar{x}^k\|^2 \\
			\leq
			& 2\max_{w \in \mathcal{R}} \|x_w^k - \tilde{x}^k\|^2 + 2 \|\frac{1}{R} \sum_{w' \in \mathcal{R}} x_{w'}^k - \tilde{x}^k\|^2 \nonumber\\
			\leq& 4\max_{w \in \mathcal{R}} \|x_w^k - \tilde{x}^k\|^2 \nonumber\\
			\leq& \frac{64(1 - \omega)^2\gamma^2 \xi^2}{\omega^2} \nonumber\\
			=& O(\frac{(\lambda + \rho)^2 \xi^2}{K(1-\lambda - 4\rho)^2}), \nonumber
		\end{align}
		which completes the proof.
	\end{proof}
	
	Theorem \ref{thm: RAgg} is inspired by the theoretical derivation in \cite{wu2023byzantine}, but with refined analysis and a simplified learning error.
	In \cite{wu2023byzantine}, the learning error is $O((1+ \frac{\hat{\lambda} + \rho \sqrt{R}}{(1- \hat{\lambda} - 8\rho\sqrt{R})^3}) (\rho^2 R + \hat{\beta}^2) \xi^2)$, in which $\hat{\lambda} \triangleq \|(I - \frac{1}{R}\bm{1}_R \bm{1}_R^{\top})M\|^2$ and $\hat{\beta} \triangleq \frac{1}{\sqrt{R}} \|M^{\top} \bm{1}_R - \bm{1}_R\|$. In contrast, our learning error is $O((\frac{\rho^2}{(1-\lambda-4\rho)^2} + \beta^2) \xi^2)$. Note that both $\lambda$ and $\hat{\lambda}$ characterize the sparsity level of the regular network, while both $\beta$ and $\hat{\beta}$ quantify the non-doubly stochasticity of the virtual mixing matrix $M$. The refinement stems from fully utilizing the row-stochasticity of the virtual mixing matrix $M$. Specifically, inspired by earlier works on decentralized optimization with row-stochastic combination weights \cite{chen2015on,sayed2014adaptation,vlaski2021distributed}, with the help of the Perron vector of $M$, we analyze the consensus to the weighted centroid $\tilde{x}^k = \sum_{w \in \mathcal{R}} p_w x_w^k$, which is better suited for analyzing decentralized optimization with a row-stochastic mixing matrix. Instead, \cite{wu2023byzantine} analyze the consensus to the uniform average $\bar{x}^k = \frac{1}{R} \sum_{w \in \mathcal{R}} x_w^k$.

	\section*{Additional Experiments}

	\begin{figure}
		\centering
		\includegraphics[scale=0.35]{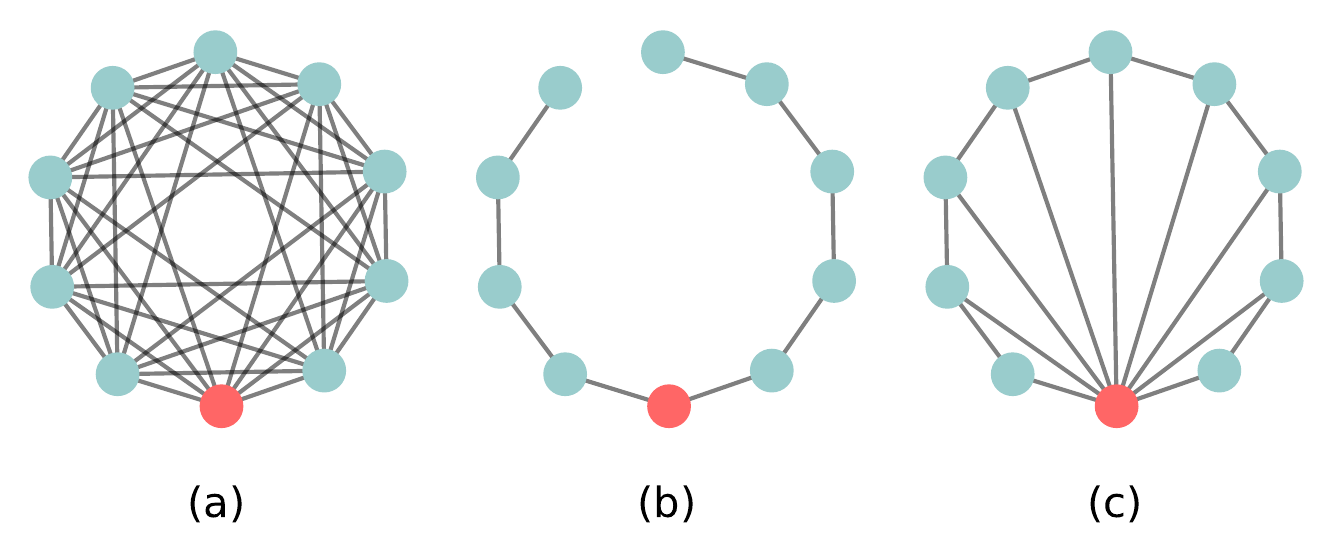}
		\caption{(a) Two-castle graph. (b) Line graph. (c) Fan graph. The blue points and red points represent the regular agents and the poisoned agents, respectively.}
		\label{fig:graph_topology}
	\end{figure}

		
		\begin{figure}
			\centering
			\includegraphics[scale=0.12]{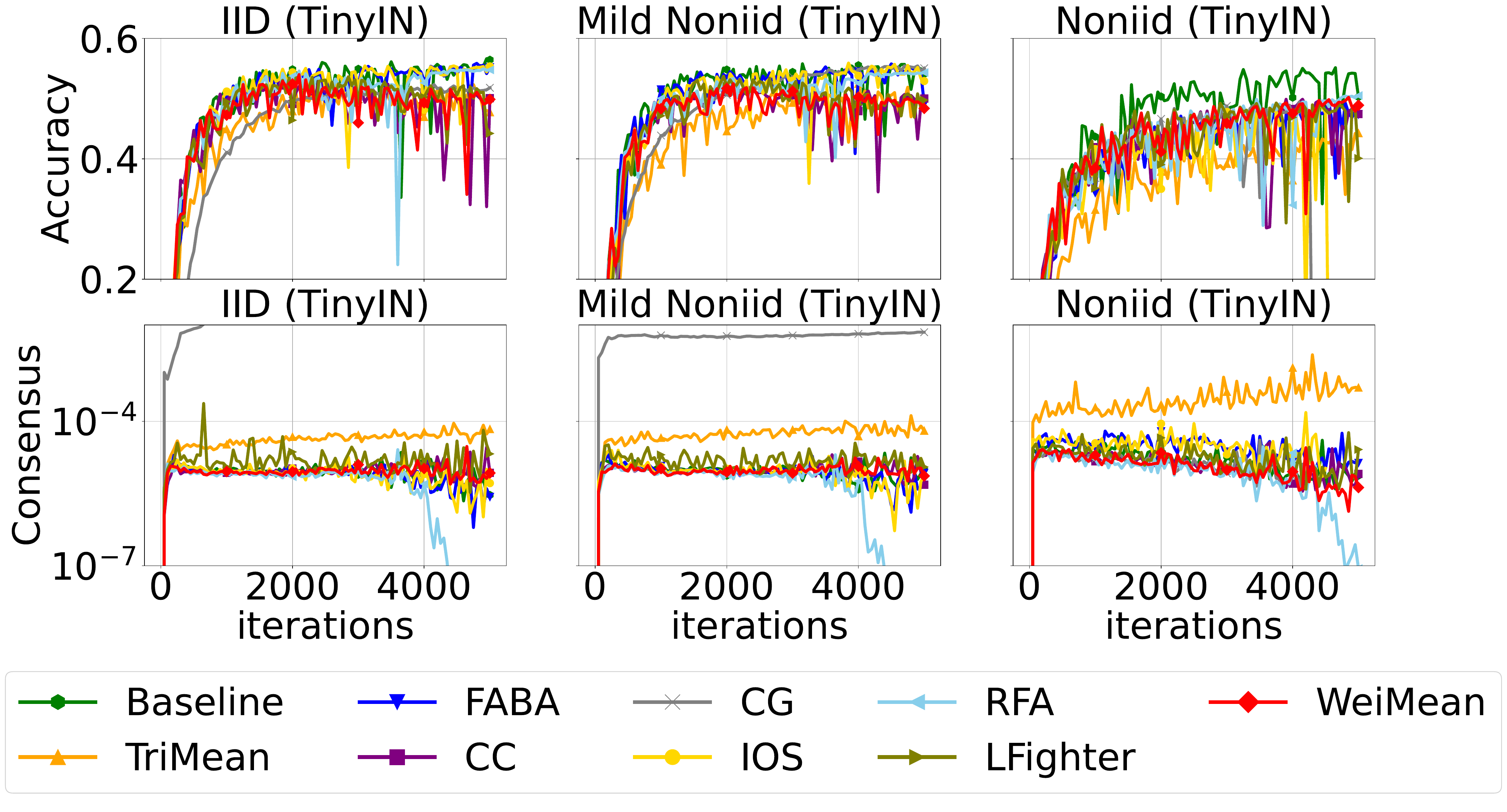}
			\caption{Classification accuracies and consensus errors of ResNet34 trained on Tiny-ImageNet (TinyIN) in the two-castle graph.}
			\label{fig:ResNet34_tinyimagenet_twocastle_label_flipping}
		\end{figure}
		
		\begin{figure}
			\centering
			\includegraphics[scale=0.14]{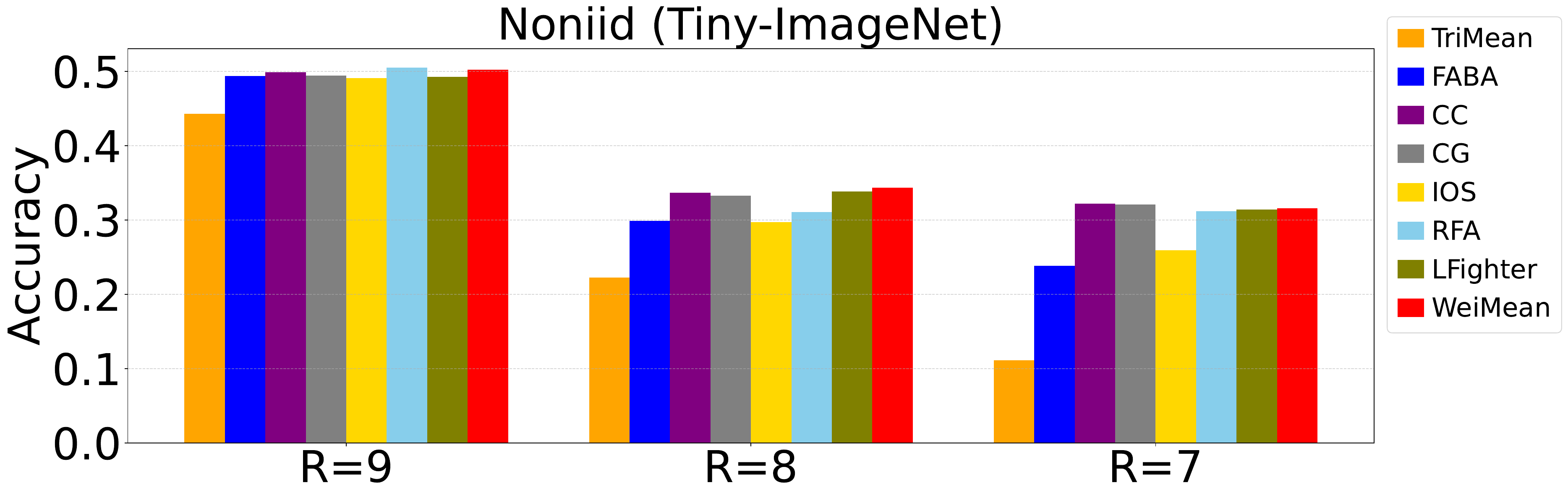}
			\caption{Classification accuracies and consensus errors of ResNet34 trained on Tiny-ImageNet in the two-castle graph, with the number of regular agents $R \in \{9, 8, 7\}$ and $W=10$, under the non-i.i.d. setting}
			\label{fig:ResNet34_tinyimagenet_twocastle_label_flipping_ablation_regular_number}
		\end{figure}
		
		\begin{figure}
			\centering
			\includegraphics[scale=0.14]{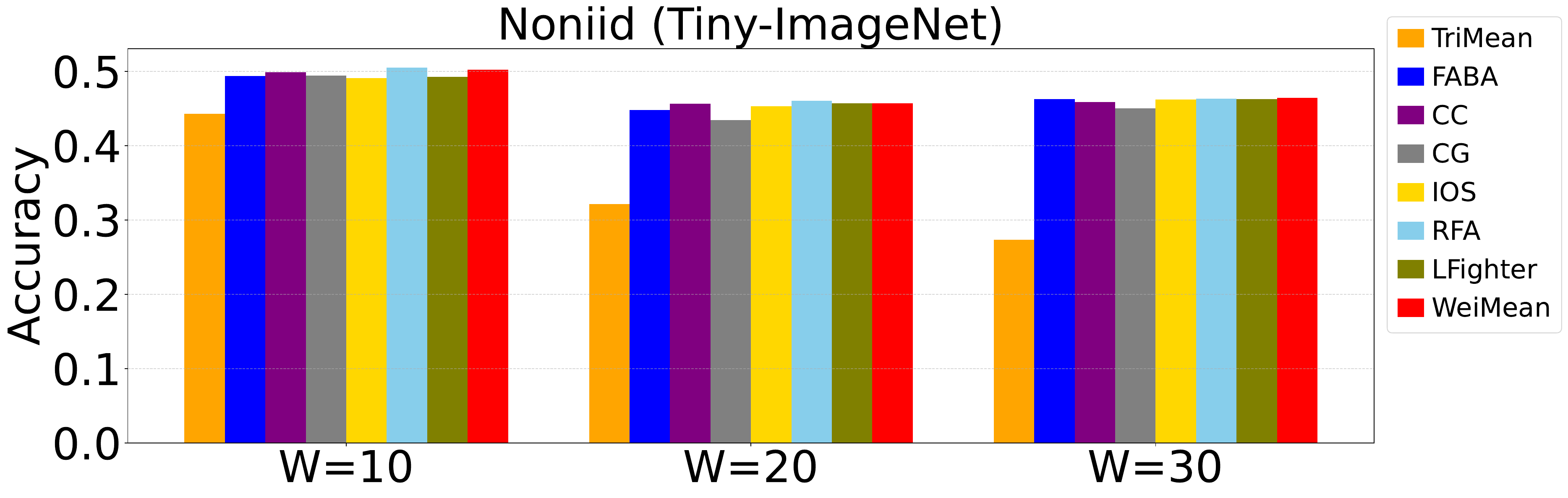}
			\caption{Classification accuracies and consensus errors of ResNet34 trained on Tiny-ImageNet in the two-castle graph, with the number of agents $W \in \{10, 20, 30\}$ and one poisoned agent, under the non-i.i.d. setting}
			\label{fig:ResNet34_tinyimagenet_twocastle_label_flipping_ablation_number}
		\end{figure}
		
		To further validate our theoretical findings, we conduct numerical experiments in broader settings. We train a ResNet34 model on Tiny-ImageNet dataset under i.i.d., mild non-i.i.d., and non-i.i.d. data partitions. The network topology is the two-castle graph shown in Figure \ref{fig:graph_topology}(a). The label poisoning attacks and compared aggregators follow the same setup as in Section \ref{sec: numerical experiments}. We set the step size to $\gamma^k = \gamma = 0.03$ and use a batch size of 32 for all aggregators.	
		
		In the two-castle graph, as discussed in Section \ref{sec: numerical experiments}, the global contamination rate $\delta = \tfrac{1}{10}$ is smaller than the local contamination rate $\delta_{\max} = \tfrac{1}{9}$. As shown in Figure \ref{fig:ResNet34_tinyimagenet_twocastle_label_flipping}, the performance of the weighted mean aggregator is worse than that of the robust aggregators in the i.i.d. and mild non-i.i.d. cases, but in the non-i.i.d. case it is among the top two, slightly worse than RFA. These empirical results validate our theoretical finding that the weighted mean aggregator outperforms the robust aggregators in sufficiently heterogeneous settings when the global contamination rate is smaller than the local contamination rate.
		
		To further verify the robustness of the weighted mean aggregator under different scalability settings of regular agents and the overall network, we vary the number of regular agents from $R=9$ to $R=8$ and $R=7$ in the two-castle graph with $W=10$ agents. In addition, we vary the total number of agents from $W=10$ to $W=20$ and $W=30$ in the two-castle graph with one poisoned agent. In all cases, the theoretical condition that the global contamination rate is smaller than the local contamination rate is satisfied. We report the accuracies of the weighted mean aggregator and  robust aggregators under non-i.i.d. data distribution in Figures \ref{fig:ResNet34_tinyimagenet_twocastle_label_flipping_ablation_regular_number} and \ref{fig:ResNet34_tinyimagenet_twocastle_label_flipping_ablation_number}. The results show that the weighted mean aggregator outperforms most robust aggregators and consistently ranks among the top three across all global contamination rates and scalability settings, demonstrating its superiority under the non-i.i.d. setting when the global contamination rate is smaller than the local contamination rate.

		\newpage
			\section*{Analysis of Majority-Dominant Aggregators}
		Within this section, we provide a theoretical analysis showing that several state-of-the-art robust aggregators, including TriMean, FABA, and IOS, satisfy Definition~\ref{def: majority-dominant} and are therefore majority-dominant aggregators. During the analysis, we encounter two special cases, namely the CC and CG aggregators, which do not possess the majority-dominant property. In particular, we construct a counterexample demonstrating that CC and CG may fail to output the same vector even when the regular inputs are identical and form a majority, provided that their clipping thresholds are chosen in a practical manner.

		For notational clarity, we define $\mathcal{H}_w \triangleq \bar{\mathcal{R}}_w  \cap \mathcal{R}$ as the set of regular inputs of agent $w \in \mathcal{W}$, with $H_w = |\mathcal{H}_w|$. Similarly, we define $\mathcal{P}_w \triangleq \bar{\mathcal{N}}_w \setminus \mathcal{H}_w$ as the set of poisoned inputs of agent $w \in \mathcal{W}$, with $P_w = |\mathcal{P}_w|$. Since the regular inputs form a majority, it follows that $P_w < H_w$. Moreover, we denote $\bar{y}_{w,R} \triangleq \frac{1}{H_w} \sum_{v \in \mathcal{H}_w} y_v$ and $\bar{y}_{w,P} \triangleq \frac{1}{P_w} \sum_{v \in \mathcal{P}_w} y_v$ as the averages of the regular and poisoned inputs of agent $w \in \mathcal{W}$, respectively.
		
		Note that the robust aggregators TriMean, FABA, and IOS require an estimate of the number of poisoned inputs to perform the aggregation step. Let $Q_w$ denote the estimated number of poisoned inputs for agent $w$. For simplicity, we set $Q_w = P_w$ in the following proofs. Following the same proof strategy, the arguments readily extend to the well-specified case where $Q_w \geq P_w$.

		\subsection{Trimmed Mean (TriMean)}
		TriMean is an aggregator that removes the largest $P_w$ and the smallest $P_w$ elements from each coordinate of the input set $\{y_v: v \in \bar{\mathcal{N}}_w \}$ for any agent $w \in \mathcal{W}$, and then computes the average of the remaining elements. Specifically, for the $i$-th coordinate, let $\pi_i$ denote a permutation on $\{1, \ldots, \bar{N}_w\}$ that sorts $\{[y_v]_i: v \in \bar{\mathcal{N}}_w\}$ in non-decreasing order, i.e., $[y_{\pi_i(1)}]_i \leq [y_{\pi_i(2)}]_i \leq \cdots \leq [y_{\pi_i(\bar{N}_w)}]_i$, where $[\cdot]_i$ denotes the $i$-th coordinate of a vector. The output is given by
		\begin{align}
			&[\text{TriMean}(\{y_v: v \in \bar{\mathcal{N}}_w\})]_i \\
			= &\frac{1}{\bar{N}_w - 2P_w} \sum_{j \in [P_w + 1, H_w]} [y_{\pi_i(j)}]_i. \nonumber
		\end{align}
		
		It has been established in \cite{guerraoui2024byzantine, zheng2025can} that TriMean is a majority-dominant aggregator. For the sake of completeness and to maintain a self-contained presentation, we provide a formal proof below.
		
		\begin{lemma}\label{lemma: majority-dominant trimean}
			TriMean is a majority-dominant aggregator.
		\end{lemma}
		
		\begin{proof}
			Since all regular inputs are identical to vector $z$, for any agent $w \in \mathcal{W}$ and any  coordinate $i$, the set $\bar{\mathcal{N}}_w$ can be partitioned into three subsets $\mathcal{A}$, $\mathcal{B}$ and $\mathcal{C}$, formally defined as:
			\begin{align*}
				&\mathcal{A} \triangleq \{v \mid v \in \bar{\mathcal{N}}_w, [y_v]_i < [z]_i\},\\
				&\mathcal{B} \triangleq \{v \mid v \in \bar{\mathcal{N}}_w, [y_v]_i = [z]_i\},\\
				&\mathcal{C} \triangleq \{v \mid v \in \bar{\mathcal{N}}_w, [y_v]_i > [z]_i\}.
			\end{align*}
			The cardinalities of these sets satisfy $|\mathcal{A}| \le P_w$, $|\mathcal{B}| \ge H_w = \bar{N}_w - P_w$ and $|\mathcal{C}| \le P_w$. Let $\mathcal{D}$ denote the set of neighboring agents whose value are not filtered by TriMean, i.e.,
			\begin{align*}
				\mathcal{D} \triangleq \{v \mid v \in \bar{\mathcal{N}}_w, [y_{\pi_i(P_w + 1)}]_i \leq [y_v]_i \leq [y_{\pi_i(H_w)}]_i\}.
			\end{align*}
			Since $|A| \leq P_w$ and $|C| \leq P_w$, it follows that  $[y_{\pi_i(P_w + 1)}]_i = [y_{\pi_i(H_w)}]_i = [z]_i$ and therefore $\mathcal{D} \subset \mathcal{B}$. Consequently, for every agent $v \in \mathcal{D}$, the $i$-th coordinate of its vector satisfies $[y_v]_i = [z]_i$. Hence,
			\begin{align}
				[\text{TriMean}(\{y_v: v \in \bar{\mathcal{N}}_w\})]_i =  \frac{1}{\bar{N}_w - 2P_w} \sum_{v \in \mathcal{D}} [y_v]_i = [z]_i.
			\end{align}
			Since this equality holds for every coordinate $i$, we obtain
			\begin{align}
				\text{TriMean}(\{y_v: v \in \bar{\mathcal{N}}_w\}) = z
			\end{align}
			which completes the proof.
		\end{proof}

		\subsection{FABA}
		FABA is an aggregator that iteratively removes the vector with the farthest Euclidean distance from the average of the current set of inputs, and then computes the average of the remaining vectors after $P_w$ iterations. Formally, let $\mathcal{U}_w^{(t)}$ denote the set of neighboring agents that remain after the $t$-th iteration for agent $w \in \mathcal{W}$. The procedure is initialized with $\mathcal{U}_w^{(0)} = \bar{\mathcal{N}}_w$. At iteration $t$, FABA computes the average of the vectors in $\mathcal{U}_w^{(t)}$ and discards the vector  farthest from this average, thereby forming $\mathcal{U}_w^{(t+1)}$. After $P_w$ iterations, FABA yields $\mathcal{U}_w^{(P_w)}$ consisting of $H_w$ agents, and outputs
		\begin{align}
			\text{FABA}\big(\{y_v: v\in \bar{\mathcal{N}}_w \}\big)
			= \frac{1}{H_w} \sum_{v \in \mathcal{U}_w^{(P_w)}} y_v .
		\end{align}
		
		It has been established in \cite{zheng2025can} that FABA is a majority-dominant aggregator. For the sake of completeness and to ensure a self-contained presentation, we provide a formal proof of the majority-dominant property of FABA below.

		\begin{lemma}\label{lemma: majority-dominant faba}
			FABA is a majority-dominant aggregator.
		\end{lemma}
		
		\begin{proof}
			Since all regular inputs are identical to  vector $z$, we have $\bar{y}_{w, R} = z$. Further, observe that
			\begin{align}
				\frac{1}{\bar{N}_w} \sum_{v \in \mathcal{U}_w^{(0)}} y_v =  \frac{H_w}{\bar{N}_w} \cdot z+ \frac{P_w}{\bar{N}_w} \cdot \bar{y}_{w, P}.
			\end{align}
			For the regular neighbor $v \in \mathcal{H}_w$, in iteration $0$, we have
			\begin{align}
				&\|y_v - \frac{1}{\bar{N}_w} \sum_{v' \in \mathcal{U}_w^{(0)}} y_{v'}\| \\
				= &\|z - (\frac{H_w}{\bar{N}_w} \cdot z + \frac{P_w}{\bar{N}_w} \cdot \bar{y}_{w, P})\| \nonumber\\
				= & \frac{P_w}{\bar{N}_w}\|z - \bar{y}_{w, P} \| \nonumber\\
				< & \frac{H_w}{\bar{N}_w} \|z - \bar{y}_{w, P}\| \nonumber \\
				=& \|\bar{y}_{w, P} - (\frac{H_w}{\bar{N}_w} \cdot z + \frac{P_w}{\bar{N}_w} \cdot \bar{y}_{w, P})\| \nonumber\\
				=& \|\bar{y}_{w, P} - \frac{1}{\bar{N}_w} \sum_{v' \in \mathcal{U}_w^{(0)}} y_{v'}\| \nonumber\\
				=& \|\frac{1}{P_w} \sum_{v' \in \mathcal{P}_w} y_{v'} - \frac{1}{\bar{N}_w} \sum_{v'' \in \mathcal{U}_w^{(0)}} y_{v''}\| \nonumber\\
				\leq &\frac{1}{P_w} \sum_{v' \in \mathcal{P}_w} \|y_{v'} - \frac{1}{\bar{N}_w} \sum_{v'' \in \mathcal{U}_w^{(0)}} y_{v''}\| \nonumber\\
				\leq & \max_{v' \in \mathcal{P}_w} \|y_{v'} - \frac{1}{\bar{N}_w} \sum_{v'' \in \mathcal{U}_w^{(0)}} y_{v''}\|, \nonumber
			\end{align}
			where the first inequality comes from that the regular inputs form a majority such that $P_w < H_w$.
			
			Therefore, in iteration 0, FABA will discard a poisoned neighboring agent $v' \in \mathcal{P}_w$. Similar to the above derivation, we have that in any iteration $t < P_w$, FABA will discard a poisoned neighboring agent $v \in \mathcal{P}_w$. Thus, after $P_w$  iterations, $\mathcal{U}_w^{(P_w)} = \mathcal{H}_w$. Therefore,
			\begin{align}
				\text{FABA}(\{y_v: v\in \bar{\mathcal{N}}_w \}) = \frac{1}{H_w} \sum _{v \in \mathcal{H}_w} y_v = z
			\end{align}
			which completes the proof.
		\end{proof}

		\subsection{IOS}
		IOS can be interpreted as the decentralized variant of FABA, where the standard average is replaced by a weighted average. Specifically, let $\mathcal{U}_w^{(t)}$ denote the set of neighboring agents that remain after the $t$-th iteration for agent $w \in \mathcal{W}$. The procedure is initialized with $\mathcal{U}_w^{(0)} = \bar{\mathcal{N}}_w$. At iteration $t$, IOS computes the weighted average of the vectors in $\mathcal{U}_w^{(t)}$ (weighted by the mixing matrix $E$), and discards the vector farthest from this weighted average, thereby forming $\mathcal{U}_w^{(t+1)}$. After $P_w$ iterations, IOS yields $\mathcal{U}_w^{(P_w)}$ consisting of $H_w$ agents, and outputs
		\begin{align}
			\text{IOS}\big(\{y_v: v\in \bar{\mathcal{N}}_w \}\big)
			= \frac{\sum_{v \in \mathcal{U}_w^{(P_w)}}E_{wv} y_v}{\sum_{v \in \mathcal{U}_w^{(P_w)}}E_{wv}}.
		\end{align}
		
		Note that in the original version of IOS, the local vector $y_w$ of agent $w$ is never removed during any iteration, so that $w \in \mathcal{U}_w^{(t)}$ for all $t \leq P_w$. Here, we consider a slight variant in which the local vector may also be removed during the iterations, since poisoned agents also employ IOS for aggregation and their local vectors cannot be fully trusted.
		
		The doubly stochastic mixing matrix $E$, used in the IOS procedure, is commonly constructed using either the Metropolis-Hastings weight (MH-weight) or the Equal-weight scheme \cite{he2022byzantine}, which are formally defined as follows:
		\begin{align}
			\hspace{-1.5em} \text{MH-weight: } E_{wv} &= \left\{
			\begin{aligned}\label{eq: Metropolis-Hastings}
				&\frac{1}{\max\{N_w, N_v\} \!+ \!1}, && v \in \mathcal{N}_w, \\
				& 1 - \sum_{l \in \mathcal{N}_w} E_{wl}, && w = v, \\
				&0, && \text{Otherwise},
			\end{aligned}
			\right. \\
			\hspace{-1.5em} \text{Equal-weight: }  E_{wv}& = \left\{
			\begin{aligned}\label{eq: Equal-weight}
				&\frac{1}{d_{\max} + 1}, && v \in \mathcal{N}_w, \\
				& 1 - \frac{N_w}{d_{\max} + 1}, && w = v, \\
				&0, && \text{Otherwise}.
			\end{aligned}
			\right.
		\end{align}
		where $d_{\max} \triangleq \max\{N_w: w \in \mathcal{W}\}$ is the maximum degree of agents in the network.
		
		Below, we show that IOS is a majority-dominant aggregator provided that, for any agent $w \in \mathcal{W}$, the total weight assigned to its poisoned inputs is less than $\frac{1}{2}$.
		
		\begin{lemma}\label{lemma: majority-dominant ios}
			If, for any agent $w \in \mathcal{W}$, the total weight of its poisoned inputs satisfies $\sum_{v \in \mathcal{P}_w} E_{wv} < \frac{1}{2}$, then IOS is a majority-dominant aggregator.
		\end{lemma}
		
		\begin{proof}
			Denoting $z' \triangleq \frac{\sum_{v \in \mathcal{P}_w} E_{wv} y_{v}}{\sum_{v \in \mathcal{P}_w} E_{wv}}$, for any agent $w \in \mathcal{W}$, we have that
			\begin{align}
				\sum_{v \in \mathcal{U}_w^{(0)}} E_{wv} y_v = \sum_{v \in \mathcal{H}_w} E_{wv} \cdot z + \sum_{v' \in \mathcal{P}_w} E_{wv'} \cdot  z'
			\end{align}
			For the regular neighbor $v \in \mathcal{H}_w$, in iteration $0$, we have
			\begin{align}
				&\|y_v - \sum_{v' \in \mathcal{U}_w^{(0)}} E_{wv'} y_{v'}\| \\
				= &\|z - ( \sum_{v' \in \mathcal{H}_w} E_{wv'} \cdot z + \sum_{v'' \in \mathcal{P}_w} E_{wv''} \cdot  z')\| \nonumber\\
				= & \sum_{v' \in \mathcal{P}_w} E_{wv'}\|z - z' \| \nonumber\\
				< & \sum_{v' \in \mathcal{H}_w} E_{wv'} \|z - z'\| \nonumber \\
				=& \|z' -  ( \sum_{v' \in \mathcal{H}_w} E_{wv'} \cdot z + \sum_{v'' \in \mathcal{P}_w} E_{wv''} \cdot  z')\| \nonumber\\
				=& \|z' - \sum_{v' \in \mathcal{U}_w^{(0)}} E_{wv'} y_{v'}\| \nonumber\\
				=& \| \frac{\sum_{v' \in \mathcal{P}_w} E_{wv'} y_{v'}}{\sum_{v' \in \mathcal{P}_w} E_{wv'}} -  \sum_{v''' \in \mathcal{U}_w^{(0)}} E_{wv'''} y_{v'''}\| \nonumber\\
				\leq &\sum_{v' \in \mathcal{P}_w}\frac{E_{wv'}}{\sum_{v'' \in \mathcal{P}_w} E_{wv''}}  \|y_{v'} -  \sum_{v''' \in \mathcal{U}_w^{(0)}} E_{wv'''} y_{v'''}\| \nonumber\\
				\leq & \max_{v' \in \mathcal{P}_w} \|y_{v'} -  \sum_{v'' \in \mathcal{U}_w^{(0)}} E_{wv''} y_{v''}\| \nonumber
			\end{align}
			where the second equality is due to the doubly stochasticity of the mixing matrix $E$, resulting in $\sum_{v \in \mathcal{H}_w } E_{wv} + \sum_{v \in \mathcal{P}_w} E_{wv} = 1$, and the first inequality comes from  $\sum_{v \in \mathcal{P}_w} E_{wv} < \frac{1}{2}< \sum_{v \in \mathcal{H}_w} E_{wv}$.
			
			Therefore, in iteration 0, IOS will discard a poisoned neighbor $v \in \mathcal{P}_w$. Similar to the above derivation, we have that in any iteration $t < P_w$, IOS will discard a poisoned neighbor $v \in \mathcal{P}_w$. Thus, after $P_w$  iterations, $\mathcal{U}_w^{(P_w)} = \mathcal{H}_w$. Therefore,
			\begin{align}
				\text{IOS}\big(\{y_v: v\in \bar{\mathcal{N}}_w \}\big)
				= \frac{\sum_{v \in \mathcal{H}_w}E_{wv} \cdot z}{\sum_{v \in \mathcal{H}_w}E_{wv}} = z.
			\end{align}
			which completes the proof.
		\end{proof}
		
		Note that the condition $\sum_{v \in \mathcal{P}_w} E_{wv} < \frac{1}{2}$ for all $w \in \mathcal{W}$ can be satisfied in practical applications when the regular inputs form a majority. For instance, if the mixing matrix $E$ is constructed using either the MH-weight or the Equal-weight schemes, then when the network topology is a complete graph, for any agent $w \in \mathcal{W}$, we have $E_{wv} = \frac{1}{\bar{N}_w}$, and thus
		\begin{align}
			\sum_{v \in \mathcal{P}_w} E_{wv} = \frac{P_w}{\bar{N}_w} < \frac{1}{2}.
		\end{align}
		If the network topology is not complete, for a regular agent $w \in \mathcal{R}$, we have for the MH-weight mixing matrix
		\begin{align}
			&\sum_{v \in \mathcal{P}_w} E_{wv} = \sum_{v \in \mathcal{P}_w} \frac{1}{\max\{N_w, N_v\} + 1} \\ &\leq \sum_{v \in \mathcal{P}_w} \frac{1}{\bar{N}_w} = \frac{P_w}{\bar{N}_w} < \frac{1}{2}, \nonumber
		\end{align}
		and for the Equal-weight mixing matrix
		\begin{align}
			&\sum_{v \in \mathcal{P}_w} E_{wv} = \sum_{v \in \mathcal{P}_w} \frac{1}{d_{\max} + 1} \\ &\leq \sum_{v \in \mathcal{P}_w} \frac{1}{\bar{N}_w} = \frac{P_w}{\bar{N}_w} < \frac{1}{2}. \nonumber
		\end{align}
		Therefore, the condition on the total weight of the poisoned neighbors is realistic when the regular inputs constitute a majority.
		
		\subsection{Centered Clipping (CC)}
		CC is an aggregator that iteratively clips the vectors from the neighboring agents. For any agent $w \in \mathcal{W}$, CC starts from some vector $s_w^0$. At iteration $t$, the update rule of CC can be formulated as
		\begin{align}
			s_w^{t+1} = s_w^t + \frac{1}{\bar{N}_w} \sum_{v \in \bar{\mathcal{N}}_w} \text{CLIP}(y_v - s_w^t, \tau_w),
		\end{align}
		where
		\begin{align}\label{eq: clipping}
			&\text{CLIP}(y_v - s_w^t, \tau_w) \\
			= &\left\{
			\begin{aligned}
				&y_v - s_w^t, && \|y_v - s_w^t\| \leq \tau_w, \\
				&\frac{\tau_w}{||y_v - s_w^t||} (y_v - s_w^t), && \|y_v - s_w^t\| > \tau_w,
			\end{aligned}
			\right. \nonumber
		\end{align}
		and $\tau_w \geq 0$ is the clipping threshold. After $T$ iterations, CC outputs the last vector as
		\begin{align}
			\text{CC}(\{y_v : v \in \bar{\mathcal{N}}_w\}) = s_w^T.
		\end{align}
		
		Note that in practical implementations of CC, the starting point is typically chosen as $s_w^0 = y_w$ for any agent $w \in \mathcal{W}$, since the local model $y_w$ generally provides a good initialization. In addition, the clipping threshold $\tau_w$ is nonzero in practice, i.e., $\tau_w \neq 0$ for any agent $w \in \mathcal{W}$. This is because zero clipping would preclude communication, preventing the agent from incorporating information from its neighbors and resulting in poor generalization of the local model to other agents' data. Below, we show that one-step CC ($T = 1$), which is recommended to use in practice\cite{karimireddy2021learning}, is \textit{not} a majority-dominant aggregator when a practical starting point and clipping threshold are used.

		\begin{lemma} \label{lemma: majority-dominant cc}
			One-step CC is not a majority-dominant aggregator when the starting point and clipping threshold satisfy $s_w^0 = y_w$ and $\tau_w \neq 0$ for all $w \in \mathcal{W}$.
		\end{lemma}
		
		\begin{proof}
			We provide a counterexample to show that CC is not a majority-dominant aggregator when $s_w^0 = y_w$ and $\tau_w \neq 0$ for all $w \in \mathcal{W}$. Since, for any agent $w \in \mathcal{W}$, all its regular inputs are identical to the vector $z$, we have $y_v = z$ for all $v \in \mathcal{H}_w$. Moreover, we let all poisoned inputs be identical to a vector $a \neq z$, i.e., $y_v = a \neq z$ for all $v \in \mathcal{P}_w$. Consequently, for any agent $w \in \mathcal{W}$, $$\{y_v : v \in \bar{\cal{N}}_w\} = \{\underbrace{z, \cdots, z}_{v\in\mathcal{H}_w}, \underbrace{a, \cdots, a}_{v \in \mathcal{P}_w}\}.$$ Below, we consider two cases, namely $\tau_w \geq \|z - a\|$ and $0 < \tau_w < \|z - a\|$, to show that there exists no clipping threshold $\tau_w$ for CC under which it satisfies Definition~\ref{def: majority-dominant}.
			
			In the first case where $\tau_w \geq \|z - a\|$, if CC were majority-dominant, the output of any regular agent $w \in \mathcal{R}$ would satisfy
			\begin{equation}
				\begin{aligned}
					& z = s_w^0 + \frac{1}{\bar{N}_w} \sum_{v \in \bar{\mathcal{N}}_w} \mathrm{CLIP}(y_v - s_w^0, \tau_w) \\
					\Rightarrow& z = z + \frac{1}{\bar{N}_w} \sum_{v \in \mathcal{P}_w} (a - z) \\
					\Rightarrow& z = a,
				\end{aligned}
			\end{equation}
			which contradicts the condition that $a \neq z$.
			
			In the second case where $0 < \tau_w < \|z - a\|$, if CC were majority-dominant, the output of any regular agent $w \in \mathcal{R}$ would satisfy
			\begin{equation}
				\begin{aligned}
					&z  =  s_w^0 + \frac{1}{\bar{N}_w} \sum_{v \in \bar{\mathcal{N}}_w} \text{CLIP}(y_v - s_w^0, \tau_w) \\
					\Rightarrow& z= z + \frac{1}{\bar{N}_w} \sum_{v \in \mathcal{P}_w} \text{CLIP}(a - z, \tau_w) \\
					\Rightarrow& 0 = \frac{P_w}{\bar{N}_w} \cdot \frac{\tau_w}{\|a - z\|} (a - z)
				\end{aligned}
			\end{equation}
			which holds only if $\tau_w = 0$ or $a = z$. This contradicts the conditions $a \neq z$ and $\tau_w \neq 0$.
			
			Thus, no meaningful clipping threshold $\tau_w \neq 0$ exists for any agent $w \in \mathcal{W}$ that allows the CC aggregator to satisfy the majority-dominance property. This completes the proof.		
		\end{proof}
		
		
		\subsection{Clipped Gossip (CG)}
		CG is the decentralized variant of CC, which also performs iterative clipping of the vectors. Specifically, the output of CG is given by
		\begin{align}
			& \text{CG}(\{y_v: v \in \bar{\mathcal{N}}_w\})  \\
			= & y_w + \sum_{v \in \bar{\mathcal{N}}_w} E_{wv} \cdot \text{CLIP}(y_v - y_w, \tau_w), \notag
		\end{align}
		where $E$ is the doubly-stochastic mixing matrix and $\text{CLIP}(\cdot)$ is defined in \eqref{eq: clipping}.
		
		Similar to the above proof of CC, we next show that CG is \textit{not} a majority-dominant aggregator, when we choose a practical clipping threshold, i.e. $\tau_w \neq 0, \forall w \in \mathcal{W}$.
		
		\begin{lemma}\label{lemma: majority-dominant cg}
			CG is not a majority-dominant aggregator when the clipping threshold satisfies $\tau_w \neq 0$ for all $w \in \mathcal{W}$.
		\end{lemma}
		
		\begin{proof}
			Since CG follows a similar aggregation step as CC, we can employ the same counterexample presented above to demonstrate that CG is not a majority-dominant aggregator. As the proof mirrors that of CC, we omit the details, thereby completing the proof.
		\end{proof}

\end{document}